\documentclass[12pt]{article}
\usepackage{amsmath}
\usepackage{graphicx}
\usepackage{enumerate}
\usepackage{natbib}
\usepackage{url} 
\usepackage{graphicx,amsfonts,amsmath,amssymb,bm,natbib,booktabs,colortbl,xcolor,float,comment,subfigure,makecell,dcolumn,multirow}
\usepackage[colorlinks,
            linkcolor=blue,
            anchorcolor=blue,
            citecolor=blue]{hyperref}
\usepackage{enumerate}
\usepackage{algorithm, algorithmicx, algpseudocode}
\floatname{algorithm}{Algorithm} 
\usepackage{enumitem}
\usepackage{setspace}

\newtheorem{theorem}{Theorem}
\newtheorem{example}{Example}
\newtheorem{definition}{Definition}
\newtheorem{proof}{Proof}

\newtheorem{lemma}{Lemma}[section]
\newtheorem{assumption}{Assumption}

\newcommand{\argmin}{\operatorname{argmin}}

\newcommand{\bE}{\mathbb{E}}

\newcommand{\indep}{\perp \!\!\! \perp}

\newcommand{\blind}{0}

\begin{document}

\def\spacingset#1{\renewcommand{\baselinestretch}%
{#1}\small\normalsize} \spacingset{1}


\if0\blind
{
  \title{\bf  Fair Sufficient Representation Learning}
\author{Xueyu $\rm{Zhou}^{a}$,  Chun Yin $\rm{Ip}^{a}$,
	and
	Jian $\rm{Huang}^{a,b}$\thanks{Corresponding author.}\\
	{\small {{$\it^{a}$Department of Data Science and Artificial Intelligence, The Hong Kong Polytechnic University}}}\\
	{\small { {$\it^{b}$Department of Applied Mathematics, The Hong Kong Polytechnic University}}}
}
\date{\footnotesize \today}
  \maketitle} \fi

\if1\blind
{
  \bigskip
  \bigskip
  \bigskip
  \begin{center}
    {\LARGE\bf Fair Sufficient Representation Learning}
\end{center}
  \medskip
} \fi

\bigskip
\begin{abstract}
The main objective of fair statistical modeling and machine learning is to minimize or eliminate biases that may arise from the data or the model itself, ensuring that predictions and decisions are not unjustly influenced by sensitive attributes such as race, gender, age, or other protected characteristics.
In this paper, we introduce a Fair Sufficient Representation Learning (FSRL) method that balances sufficiency and fairness. Sufficiency ensures that the representation should capture all necessary information about the target variables, while fairness requires that the learned representation remains independent of sensitive attributes. FSRL is based on a convex combination of an objective function for learning a sufficient representation and an objective function that ensures fairness. Our approach manages fairness and sufficiency at the representation level, offering a novel perspective on fair representation learning. We implement this method using distance covariance, which is effective for characterizing independence between random variables.
We establish the convergence properties of the learned representations. Experiments conducted on healthcase and text datasets with diverse structures demonstrate that FSRL achieves a superior trade-off between fairness and accuracy compared to existing approaches.

\end{abstract}

\noindent%
{\it Keywords: Balancing sufficiency and fairness, conditional independence, convergence, deep neural networks,  representation learning.}
\vfill

\newpage
\spacingset{1.2} 
\section{Introduction}
The main goal of fair statistical modeling and machine learning is to reduce or eliminate biases that may arise from the data or the model itself, ensuring that predictions and decisions are not unfairly influenced by sensitive attributes such as race, gender, age, or other protected characteristics.
However, due to data imbalances and social biases, standard statistical and machine learning methods can result in decisions that perpetuate serious social stereotypes related to protected attributes such as gender and race.
Therefore, it is important to develop methods that accounts for data imbalance and ensuring fair predictions. Among recent advancements in fair machine learning, learning fair representations is a crucial aspect and plays a key role in building fair statistical and machine learning models \citep{edwards2015censoring,louizos2015variational,beutel2017data,guo2022learning,park2022fair,zhao2022inherent}.

Representation learning seeks to find low-dimensional representations that capture the essential information of data. The pioneering work by \cite{li1991sliced} introduced the sufficient dimension reduction (SDR) approach.
Within the SDR framework, \cite{li1991sliced} proposed a semi-parametric method known as sliced inverse regression, which estimates a linear sufficient representation (SIR). SIR, however, imposes linearity and constant covariance assumptions on the predictor distribution.
Several other methods for sufficient dimension reduction have also been developed, including those based on conditional covariance operators \citep{fukumizu2009kernel}, mutual information \citep{suzuki2013sufficient}, distance correlation \citep{vepakomma2018supervised}, and semiparametric modeling \citep{ma2012semi}.
 However, these SDR methods primarily focus on \textit{linear dimension reduction}, which may not adequately capture the complexity of high-dimensional data such as images and natural languages that exhibit highly nonlinear characteristics.
Recently, nonlinear representation learning methods have been proposed in various statistical and machine learning tasks, including supervised learning \citep{huang2024deep,chen2024deep}, self-supervised learning \citep{chen2020simple,he2020momentum}, and transfer learning \citep{neyshabur2020being,kumar2022fine}. However, while these methods focus on finding representations with strong predictive power, they often overlook the issue of fairness.

Fair representation learning aims to develop representations that balance predictive power with fairness. In previous studies, \citet{zemel2013learning} introduced a framework for learning low-dimensional representations for the original data, using demographic parity as the fairness criterion. Building on this approach, several studies \citep{louizos2015variational,liu2022fair,guo2022learning} have attempted to learn fair representations through variational autoencoders \citep[VAE]{kingma2013auto} with various fairness constraints. However, VAE  prioritizes data reconstruction capabilities over fairness and prediction. Additionally, training the decoder requires significant computational resources, especially when handling high-dimensional complex data. \citet{beutel2017data} applied adversarial training to remove information about sensitive features from the representation by maximizing the prediction error of these features. However, maximizing the error does not necessarily equate to eliminating sensitive information and achieving fair outcomes. Therefore, further work is needed to develop methods for ensuring the fairness of learned representations.

In this paper, we propose a Fair Sufficient Representation Learning (FSRL) method. Our goal with FSRL is to balance the predictive power and fairness of the learned representations. Our approach builds on the work of \citet{huang2024deep} and \citet{chen2024deep}, who explored deep learning methods for sufficient supervised representation learning tasks. We use conditional independence to characterize sufficiency \citep{li1991sliced}, encouraging the representation to encode information about target variables. In addition, we apply the concept of statistical independence to measure fairness, focusing on protecting sensitive information. By using a convex combination of the criteria for sufficiency and fairness, we construct an objective function that balances these two aspects. We use deep neural networks to approximate nonlinear representations. We establish the theoretical properties and present a non-asymptotic risk bound under suitable conditions. Furthermore, we conduct numerical studies to evaluate the performance of FSRL based on its predictive capabilities. We also use data from healthcare 
and text analysis to illustrate the applications of FSRL and demonstrate its superior ability to balance sufficiency and fairness.

The rest of the paper is organized as follows.  Section 2 presents the proposed method and describes the characteristics of FSRL in terms of sufficiency and fairness. Section 3 presents the objective function of FSRL and gives the implementation with distance covariance. Section 4 establishes its convergence analysis. Section 5 and Section 6 present the experimental results of the proposed method with adequate simulation studies and real-world datasets. Section 7 reviews some related works. Section 8 contains some discussion for this paper and potential future works.

\section{Method}
Suppose we observe a triplet of random vectors $(X,Y,A)\in \mathbb{R}^p\times \mathbb{R}^q\times \mathbb{R}^m$, where $X$ represents a vector of predictors, $Y$ is the response variable, and $A$ encodes information that should be excluded from the representation, such as private attributes like gender or race. Our objective is to learn a representation function $R: \mathbb{R}^p \mapsto \mathbb{R}^d$ with $1 \le d \le p$  that \textit{balances two different requirements}:
\begin{itemize}
    \item $R(X)$ should be sufficient to encode all the information in $X$ about $Y$.
    \item  $R(X)$ should be agnostic to sensitive information encoded in $A$.
\end{itemize}
These two requirements can be inherently conflicting. A representation that is sufficient in the traditional sense may not be agnostic to $A,$, and a fair representation may not be sufficient. Therefore, it is generally impossible to find a representation that is both entirely fair and sufficient. Instead, the goal is to achieve the desired trade-off between these two requirements.

\subsection{Fairness}
There are several definitions of fairness in the literature on fair machine learning. For example, for a binary factor
$A=0$ or $1,$
\citet{zemel2013learning} considered demographic parity by requiring the predictive value $\hat{Y}$
is independent of $A$, that is, $P\{\hat{Y}=1|A=1\}=P\{\hat{Y}=1|A=0\}.$
Moreover, an alternative fairness criterion, equalized odds \citep{hardt2016equality}, requires that $\hat{Y}$ be conditionally independent of $A$ given the target $Y$, that is $P\{\hat{Y}=1|Y=y,A=1\}=P\{\hat{Y}=1|Y=y,A=0\}, y\in\{0,1\}.$
These fairness definitions are applied in various contexts, such as assessing income disparities between genders via demographic parity and evaluating fairness in college admissions using equalized odds. In this work, we focus primarily on demographic parity and consider fairness specifically in the learned representation.

\begin{definition}\label{fairness}
    The measurable representation function $R:\mathbb{R}^p\rightarrow \mathbb{R}^d$
     is a fair representation with respect to the variable $A$ if $R(X)$ is independent of $A$, that is $R(X)\indep A$.
\end{definition}

This definition of fairness implies that the representation should remain unaffected by the influence of factor $A.$  Consequently, when such a fair representation is employed within a predictive model—be it for classification or regression—the outcomes of the predictions will remain uninfluenced by factor $A$, which satisfies demographic parity. This ensures that the model's decisions are equitable and do not inadvertently favor or disadvantage any group associated with $A.$ By maintaining this independence, the model upholds fairness, promoting trust and reliability in its predictive capabilities.

This concept is related to Invariant Risk Minimization (IRM) \citep{arjovsky2019invariant}, which seeks to enhance the generalization capabilities of machine learning models across diverse environments. The core idea of IRM is to identify a data representation such that the optimal classifier for this representation remains consistent across all training environments. This is accomplished by encouraging the model to capture invariant features that are predictive across various environments, rather than relying on spurious correlations that may not hold outside the training data. \citet{zhu2023invariant} studied invariant representation learning, which focuses on developing a representation $R(X)$ that remains consistent across different domains.
It is important to distinguish between the concept of fairness, as defined above, and the notion of invariance in IRM, as they are fundamentally different. In fair representation learning, the sensitive attribute $A$  is a random variable correlated with
$(X,Y).$. In contrast, IRM and invariant representation learning require the representation to remain unchanged across different domains.

\subsection{Sufficiency}
The concept of a sufficient representation was proposed by \cite{li1991sliced} in the context of
sufficient dimension reduction, where the goal is to learn a linear representation function in semiparametric regression models. In the more general nonparametric setting with nonlinear representation functions, a sufficient representation can be defined as follows.

\begin{definition}
    A measurable function
    $R:\mathbb{R}^p\rightarrow \mathbb{R}^d$  is called a sufficient representation of $X$
    if $X$ and $Y$ are conditional independent given $R(X)$, that is $X\indep Y|R(X)$.
\end{definition}

Sufficiency implies that all the useful information in $X$ about target $Y$ has been encoded into $R(X),$ or equivalently, the conditional distribution of $Y$ given $X$ is the same as that of $Y$ given $R(X),$ i.e., $p(y|x) = p(y|R(x)).$   A sufficient representation always exists since $R(X)=X$ trivially satisfies this requirement. The existing sufficient dimension reduction (SDR) methods focus on linear representation functions \citep{cook2005sufficient,li2007directional,zhu2010dimension}. Recent works have attempted to generalize this approach to nonparametric settings with nonlinear representations \citep{huang2024deep, chen2024deep}.  These works considered deep sufficient representation learning and demonstrated the power of deep representation approaches in analyzing complex high-dimensional data.

\subsection{Balancing sufficiency and fairness}
By leveraging sufficiency and fairness, we can characterize the information preserved in the fair representation and propose our goal for FSRL as follows:
\begin{align} &X\indep Y|R(X)\nonumber\\\
\text{subject to }&R(X)\indep A.\label{goal}
\end{align}
Our aim is to find a sufficient representation that is also fair, meaning it is independent of the sensitive attribute $A.$

However, a fundamental challenge arises from the potential conflict between these two conditions, due to the possible dependence between the target variable $Y$ and the sensitive attribute $A$. To address this conflict, it is essential to assess which information should be retained and consider how to achieve an appropriate balance.

\begin{figure}[H]\label{exam1}
\centering
\includegraphics[width=6.0 in, height=1.6 in]{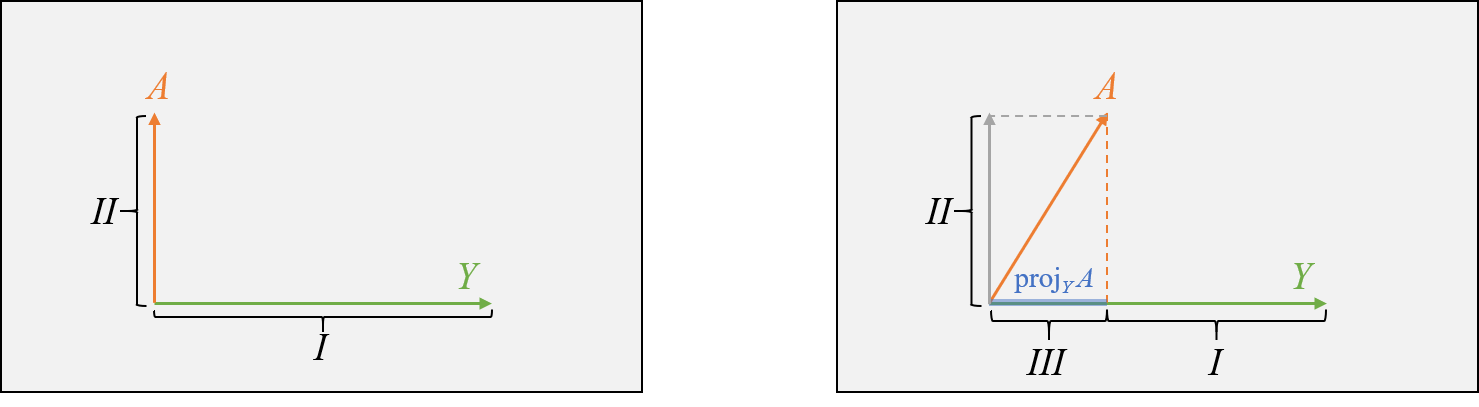}
\caption{Relationship between target variable $Y$ and sensitive attribute $A$:
Left panel:  $Y$ is independent of $A$, indicating that their information is orthogonal; right panel: $Y$ is dependent on $A$, resulting in overlapping information.}
\label{exam1}
\end{figure}

In Figure \ref{exam1}, we depict different relationships between target $Y$ and sensitive attribute $A$. When $Y$ is independent of  $A$ (left panel), the ideal representation $R(X)$ could encode all the information about $Y$ (Part I) while remaining agnostic to sensitive information $A$ (Part II). However, when $Y$ is dependent on $A$ (right panel), it becomes impossible to satisfy both conditions in \eqref{goal}, requiring \textit{a trade-off between sufficiency and fairness}. Sufficiency encourages the representation to preserve information from both Part I and Part III, whereas fairness demands the exclusion of information from Part II and Part III. After maintaining Part I and eliminating Part II, the final balance between predictive power and fairness is represented by Part III.

\section{Nonparametric Estimation of Representation}
The discussion in the previous section provides the foundation for formulating an objective function that enables FSRL to learn a representation and solve problem \eqref{goal}.

\subsection{Population objective function}
Let $\mathbb{V}$ be a measure of dependence between two random variables  $(U,V)$ with the following properties:(a)  $\mathbb{V}[U,V] \ge 0$ with $\mathbb{V}[U, V]=0$ if and only if $U\indep V$; (b) $\mathbb{V}[U, V] \ge \mathbb{V}[R(U), V]$ for every measurable function $R$; (c) $\mathbb{V}[U,V]=\mathbb{V}[R^*(U), V]$ if and only if $U\indep V|R^*(U)$. These properties imply that the objective \eqref{goal} can be achieved by $\max \mathbb{V}[R,Y]$ and $\min \mathbb{V}[R,A]$.

Since sufficiency and fairness remain invariant under one-to-one transformation, directly estimating the representation within a large space is challenging. To address this, we focus on a narrow subspace where the representation follows a standard Gaussian distribution \citep{huang2024deep}. In this subspace, each component of the representation is disentangled, encoding distinct features related to sufficiency and fairness. Our goal is then formulated as a constrained optimization problem:
\begin{align}
    \min_R -&\mathbb{V}[R(X),Y]\nonumber\\
    \text{s.t. }&\mathbb{V}[R(X),A]=0, R(x)\sim N(0,I_d).\label{constrained}
\end{align}

There are several options for $\mathbb{V}$ with the above properties. In this work, we use the distance covariance \citep{szekely2007measuring} as the dependence measure. Let $\mathfrak{i}$ be the imaginary unit $(-1)^{1 / 2}$. For any $s\in \mathbb{R}^{p}$ and $t\in \mathbb{R}^q$, let  $\phi_{U}(s)=\mathbb{E} [\exp^{\mathfrak{i}s^TU}], \phi_{V}(t)=\mathbb{E} [\exp^{\mathfrak{i}t^TV}],$ and $\phi_{U, V}(s,t) = \mathbb{E} [\exp^{\mathfrak{i}(s^T U+ t^T V)}]$ be the characteristic functions of random vectors $U\in \mathbb{R}^p , V \in \mathbb{R}^q,$ and the pair  $(U, V)$, respectively. Let $||\cdot||$ denote the Euclidean norm.
    Given a pair of random variable $(U,V)\in \mathcal{R}^p \times \mathcal{R}^q$ with finite second moments, the distance covariance between $U$ and $V$ is defined as
    \begin{align*}\label{objectivefunction}
        \mathbb{V}[U,V] &= \int_{\mathbb{R}^p\times \mathbb{R}^q} \frac{|\phi_U(t)\phi_V(s)-\phi_{(U,V)}(t,s)|}{c_pc_q||t||^{1+p}||s||^{1+q}}dtds \nonumber \\
        &=\mathbb{E}||U-U^{'}||||V-V^{'}||-2\mathbb{E}||U-U^{'}||||V-V^{''}|| \nonumber\\
        &+\mathbb{E}||U-U^{'}||\mathbb{E}||V-V^{'}||,
    \end{align*}
    where $c_p=\frac{\pi^{(1+p)/2}}{\Gamma((1+p)/2)}$. $(U^{'},V^{'}),(U^{''},V^{''})$ denote the independent copy of $(U,V)$. Distance covariance is an effective measure of both linear and nonlinear relationships between two random variables. Previous studies \citep{liu2022fair,zhen2022versatile,zhu2023invariant,huang2024deep} have demonstrated its utility across various domains.

For the Gaussian regularization, we employ energy distance \citep{szekely2013energy}. Given two random variables $U\in\mathbb{R}^d$ and $V\in\mathbb{R}^d$, the energy distance $\mathbb{D}$ is defined as
\begin{align*}
   \mathbb{D}[U,V] &= \int_{\mathbb{R}^d} \frac{|\phi_U(t)-\phi_V(t)|}{c_d||t||^{1+d}}dt \nonumber \\
    &=2\mathbb{E}||U-V||-\mathbb{E}||U-U'||-\mathbb{E}||V-V'||,
\end{align*}
where $c_d=\frac{\pi^{(1+d)/2}}{\Gamma((1+d)/2)}$ and where $\{U',V'\}$ are independent copies of $\{U,V\}$.
Energy distance satisfies the condition
$\mathbb{D}[U,V]\ge 0 $
and  $\mathbb{D}[U,V]=0$ if and only if
$U$ and $V$ have the same distribution, which can be used to characterize the Gaussian regularization.

To solve the constrained problem \eqref{constrained}, we apply the Lagrangian method and use the equivalent convex combination of $V[R(X),Y]$ and $V[R(X),A]$. Our objective function is
\begin{equation}\label{objective}
    \mathcal{L}(R) =  -\alpha \mathbb{V}[R(X),Y] + (1-\alpha)\mathbb{V}[R(X),A] + \lambda \mathbb{D}[R(X),\gamma_d], 0<\alpha\leq 1,0\leq \lambda,
\end{equation}
where $\gamma_d$ follows a standard Gaussian distribution. In objective function \eqref{objective}, we introduce the parameter $\alpha\in(0,1]$ to control the balance between sufficiency and fairness. As $\alpha$  approaches 1, the representation will retain more private information from part III of Figure \ref{exam1} (right panel),  leading to a more unfair representation. The tuning parameter $\lambda$ controls the Gaussianity regularization.
\begin{theorem}\label{theorem1}
    If $Y\indep A$, for any $\alpha\in(0,1]$ and $\lambda\geq 0$, we have $R_0\in \argmin_{R\sim N(0,I_d)}\mathcal{L}(R)$ provides \eqref{goal} holds.
\end{theorem}
Theorem \ref{theorem1} shows that if $Y$ and $A$ are independent, we can achieve both both sufficiency and fairness for any values of  $\alpha \in (0, 1]$ and $\lambda \geq 0$. The proof of theorem \ref{theorem1} is provided in Section A of the Supplementary Material.

\subsection{Gaussian linear case}
When the response variable $Y$ is not independent of the sensitive attribute $A$, there is an inherent conflict between sufficiency and fairness. In this case, we present a linear example that demonstrates the mechanism of the proposed FSRL framework.

    Suppose that $X\in\mathbb{R}^p$ follows a multivariate Gaussian distribution with mean vector $0$ and covariance matrix $I_p$. The target variable $Y$ is defined as $Y = f(P^TX) + \epsilon_Y$, and the sensitive attribute $A$ is defined as $A = g(Q^TX) + \epsilon_A$, where $\epsilon_Y$ and $\epsilon_A$ are random noises independent of $X$. $P\in\mathbb{R}^{p\times d_Y}$, $Q\in\mathbb{R}^{p\times d_A}$ and $P^TP=I_{d_Y}$, $Q^TQ=I_{d_A}$. We aim to find a linear representation $R^TX$ that is fair to $A$ and maintains enough information about $Y$.

An ideal solution 
is given by $R=(I - Q(Q^TQ)^{-1}Q^T)P$, which corresponds to part I in Figure \ref{exam1}. On the one hand, $R^TX$ is fair to $A$ as it removes all sensitive information with $span(R) \perp span(Q)$. On the other hand, $R$ retains all the relevant information about $Y$ after removing the sensitive information.

Under the linear model, the Gaussianity regularization is equivalent to $R^TR=I_{d_0}$, where $d_o$ is the rank of matrix $(I - Q(Q^TQ)^{-1}Q^T)P$. The constrained problem \eqref{constrained} is
\begin{align}
        \min_{R^TR=I_{d_0}} -&\mathbb{V}[R^TX,Y] \nonumber\\
        \text{s.t. } &\mathbb{V}[R^TX,A]=0. \label{linear}
    \end{align}
Based on Propositions 1 and 2 in \citet{sheng2016sufficient}, the solution $\tilde{R}$ of problem \eqref{linear} is capable of accurately identifying the desired subspace and $span(\tilde{R})=span((I - Q(Q^TQ)^{-1}Q^T)P)$. Our approach balances the two competing goals of preserving essential information in the representation and maintaining fairness.

\subsection{Empirical objective function}
In this subsection, we formulate the empirical objective function and estimate the representation through neural networks. First, we introduce the empirical estimation of distance covariance and energy distance. Given random samples samples $\{u_i,v_i\}_{i=1}^n$, the $U$-statistic  \citep{huo2016fast} of distance covariance $\mathbb{V}[U,V]$ is
\begin{align}\label{empdcov}
     \mathbb{V}_n[U,V]=\frac{1}{C_n^4}\sum\limits_{1\leq i_1 < i_2 < i_3 < i_4\leq n}h((u_{i_1},v_{i_1}),\cdots,(u_{i_4},v_{i_4})),
 \end{align}
 where $h$ is the kernel function defined as
 \begin{align*}
     h((u_{i_1},v_{i_1}),\cdots,(u_{i_4},v_{i_4})) &=\frac{1}{4}\sum\limits_{1\leq i,j \leq 4 \atop i\neq j}\vert\vert u_i-u_j\vert\vert\vert\vert v_i-v_j\vert\vert \nonumber\\
     &+\frac{1}{24}\sum\limits_{1\leq i,j \leq 4 \atop i\neq j}\vert\vert u_i-u_j\vert\vert\sum\limits_{1\leq i,j \leq 4 \atop i\neq j}\vert\vert v_i-v_j\vert\vert \nonumber\\
     &-\frac{1}{4}\sum\limits_{i=1}^4(\sum\limits_{1\leq i,j \leq 4 \atop i\neq j}\vert\vert u_i-u_j\vert\vert\sum\limits_{1\leq i,j \leq 4 \atop i\neq j}\vert\vert v_i-v_j\vert\vert).
 \end{align*}
 Moreover, as suggested by \citet{gretton2012kernel}, given random samples $\{u_i\}_{i=1}^n$ and $\{v_i\}_{i=1}^n$, the empirical version of energy distance $\mathbb{D}[U,V]$ is
\begin{align}
    \label{edn}
\mathbb{D}_n(U,V) = \frac{1}{C_n^2} \sum_{1\leq i,j\leq n}g(u_i,u_j;v_i,v_j),
\end{align}
where
$g(u_i,u_j;v_i,v_j)= \|u_i- v_j\| + \|u_j- v_i\| - \|u_i- u_j\| - \|v_i- v_j\|.$

Based on \eqref{empdcov} and \eqref{edn}, we can formulate the empirical version of the objective function \eqref{objective}. Given random samples $\{(X_i,Y_i,A_i)\}_{i=1}^n$ of $(X,Y,A)$ and $\{\gamma_{di}\}_{i=1}^n$ sampled from random variable $\gamma_d \sim N(0,I_d)$, the empirical objective function is
\begin{align}\label{empirical}
    \mathcal{L}_n(R) =  -\alpha \mathbb{V}_n[R(X),Y] + (1-\alpha)\mathbb{V}_n[R(X),A] + \lambda\mathbb{D}_n[R(X),\gamma_d], 0<\alpha\leq 1,0\leq\lambda.
\end{align}

Due to the significant approximation capabilities \citep{shen2020deep,jiao2023deep}, we use feedforward neural networks with Rectified Linear Unit \citep[ReLU]{glorot2010understanding} to estimate the representation $R$.
Let $\mathbf{R}_{\mathcal{D,W,S,B}}$ be the set of such ReLU neural networks $R:\mathbb{R}^p \rightarrow \mathbb{R}^d$ with depth $\mathcal{D}$, width $\mathcal{W}$, size $\mathcal{S}$ and boundary $\vert\vert R \vert\vert_\infty \leq B$. The depth $\mathcal{D}$ refers to the number of hidden layers and the network has $\mathcal{D}+2$ layers in total, including the input layer, hidden layers, and the output layer. $(w_0,w_1,...,w_\mathcal{D},w_{\mathcal{D}+1})$ specifies the width of each layer, where $w_0=p$ is the dimension of input and $w_{\mathcal{D}+1}=d$ is the dimension of the representation. The width $W=\max\{w_1,...,w_\mathcal{D}\}$ is the maximum width of hidden layers and the size $S=\sum_{i=0}^{\mathcal{D}}[w_{i+1}\times(w_i+1)]$ is the total number of parameters in the network. For the set $\mathbf{R}_{\mathcal{D,W,S,B}}$, its parameters satisfy the simple relationship
 \begin{align*}
     \max\{W,D\}\leq S \leq W(p+1)+(W^2+W)(D-1)+(W+1)d=O(W^2D).
 \end{align*}
Based on the empirical objective function and neural network approximation, we have the nonparametric estimation with given $\alpha \in (0,1]$ and $\lambda\geq 0$.
\begin{align*}
    \hat{R}=\argmin_{R \in \mathbf{R}_{\mathcal{D,W,S,B}}} -\alpha \mathbb{V}_n[R(X),Y] + (1-\alpha)\mathbb{V}_n[R(X),A]+\lambda \mathbb{D}_n[R(X),\gamma_d].
\end{align*}
After estimating the representation $\hat{R}$, we can make predictions for the downstream tasks based on the representation. The whole procedure is summarized in Algorithm \ref{algo_FSRL}.
\begin{algorithm}[!ht]
\footnotesize
\spacingset{1.3}
\footnotesize
\spacingset{1.3}
\caption{Sufficient and Fair Representation Learning} 
  \label{algo_FSRL}
  \textbf{Input: }
  Random samples $\{(x_i,y_i,a_i)\}_{i=1}^n$ and $\{\gamma_{di}\}_{i=1}^n$ sampled from $\gamma_d \sim N(0,I_d)$.

\textbf{Step I: Estimating representation $R$ through FSRL}
  \begin{algorithmic}[1]
    \State
    With the random samples $\{(x_i,y_i,a_i)\}_{i=1}^n$ and $\{\gamma_{di}\}_{i=1}^n$,
    we learn $R$ with the following objectives,
\begin{align*}
  \hat  R = \argmin_{R \in \mathbf{R}_{\mathcal{D,W,S,B}}} \mathcal{L}_n(R) ,
\end{align*}
where
$\mathcal{L}_{n}(R)$ is defined in \eqref{empirical}.
  \end{algorithmic}

  \textbf{Step II: Estimating the prediction function $f_{R}$ for downstream task}
  \begin{algorithmic}[1]
  \State
With random samples $\{(x_i,y_i)\}_{i=1}^n$ and frozen representation $\hat{R}$, we study the the prediction function $f_R$ with the following objectives,
\begin{align*}
    \hat{f}_R = \argmin_{f\in \mathcal{F}} \mathcal{L}_{task,n}(f(\hat{R}(X),Y),
\end{align*}
where
$\mathcal{F}$ is the function class for $f_R$ and $\mathcal{L}_{task,n}$ is the empirical loss function for downstream tasks, such as MSE for regression or cross-entropy loss for classification.

\end{algorithmic}

\textbf{
  Output: The estimated representation $\hat{R}$ and prediction function $\hat f_R$.
}
\end{algorithm}

\section{Convergence Analysis}
In this section, we will build the excess risk bound for the learned representation estimated through deep neural networks under mild conditions.
We begin by defining $\beta$-H{\"o}lder smoothness and outlining some mild assumptions, including the smoothness of representation function, data distributions and neural network parameters.
 \begin{definition}
      A $\beta$-H{\"o}lder smooth class $\mathcal{H}^\beta([0,1]^p,B)$ with $\beta=k+a$, $k\in \mathbb{N}^+,a\in(0,1]$ and a finite constant $B>0$,is a function class consisting of function $f:[0,1]^p\rightarrow R$ satisfying
      \begin{align*}
      \max\limits_{||\alpha||_1\leq k} ||\partial^\alpha f||_\infty \leq B, \max\limits_{||\alpha||_1= k} \max_{x\neq y}\frac{|\partial^\alpha f(x)-\partial^\alpha f(y)|}{||x-y||_2^a} \leq B,
      \end{align*}
    where $||\alpha||_1=\sum_{i=1}^p\alpha_i$ and $\partial^\alpha=\partial^{\alpha_1}\partial^{\alpha_2}\cdots\partial^{\alpha_p}$ for $\alpha=(\alpha_1,...,\alpha_p)$.
 \end{definition}
Denote $R_0(X)=\argmin\mathcal{L}(R)$.
   \begin{assumption}\label{Ass1}
      Each component of $R_0(X)$ is $\beta$-H{\"o}lder smooth on $[0,1]^p$ with constant $\mathcal{B_0}$, that is $R_{0,i}\in \mathcal{H}^\beta([-0,1]^p,B_0), i=1,...,d$.
 \end{assumption}

  \begin{assumption}\label{Ass2}
      The support of $X$ is contained in a compact set $[0,1]^p$. $Y$ and $A$ are bounded almost surely: $||Y||\leq B_1$ , $||A||\leq B_2$, a.s..
 \end{assumption}

 \begin{assumption}\label{Ass3}
     Representation network $\mathbf{R}\equiv \mathbf{R}_{\mathcal{D,W,S,B}}$ parameters: \\
     depth $D= \mathcal{O}(\log n\log_2(\log n))$, width $\mathcal{W} = \mathcal{O}(n^{\frac{p}{2(2\beta+p)}}\log_2(n^{\frac{p}{2(2\beta+p)}}/\log n)/\log n)$, size $\mathcal{S} = \mathcal{O}(dn^{\frac{p}{2\beta+p}}/\log^4(npd))$, boundary $\sup_{X\in[0,1]^p}||R(X)||_2\leq\mathcal{B}$.
 \end{assumption}
The boundary conditions of $Y$ and $A$ in assumption \ref{Ass2} are commonly satisfied in fair machine learning. We emphasize that the network parameters in Assumption \ref{Ass3} are not necessarily unique and other assumptions about smoothness and network parameters can be made to obtain a similar excess risk bound.
 \begin{theorem}[No-asymptotic error bound]\label{Convergence}
      Suppose assumption \ref{Ass1},\ref{Ass2} and \ref{Ass3} hold, then
  \begin{align}\label{bound}
     &\mathbb{E}_{ \{X_i,Y_i,A_i\}_{i=1}^n}\{L(\hat{R})-L(R_0)\} \leq \mathcal{O}(n^{\frac{-\beta}{2\beta+p}}).
 \end{align}
 \end{theorem}
Theorem \ref{Convergence} shows that FSRL is consistent with an appropriately selected network structure as n $\rightarrow \infty$.
We omit the coefficients that consist of $p$, $d$, and $\beta$ in \eqref{bound}.
The proof of Theorem \ref{Convergence} and detailed bound are given in Section A of the Supplementary Materials.

\section{Simulation Studies}
In this section, we evaluate the performance of the FSRL method through two simulation studies.  For all experiments, we first use FSRL to learn the representation and then study a prediction function with the frozen representation as input. A total of 10,000 training samples, 1,000 validation samples, and 1,000 testing samples are generated independently in each example. For simplicity, both the representation function and the prediction function are trained on the full training dataset and validated using the validation dataset without separation. The numerical performance of FSRL is subsequently assessed on the testing dataset. Moreover, we construct a deep neural network (\textbf{DNN}) whose structure is a direct splicing of the representation function and prediction model without any fair constraints as the unfair baseline. To measure the fairness of different methods, we focus on demographic parity \cite{zemel2013learning} and implement $\Delta DP$ as the fair criterion, which is defined as
\begin{align*}
    \Delta DP &=\max_{y\in \mathcal{M},i\neq j\in\mathcal{N}} |P\{\hat{Y}=y|A=i\}-P\{\hat{Y}=y|A=j\}|,
\end{align*}
where $\mathcal{M}$ is the collection of all possible values of $Y$ and $\mathcal{N}$ is the collection of $A$.
$\Delta DP$ is influenced by the dependence between the prediction $\hat{Y}$ and sensitive attribute $A$  and a smaller $\Delta DP$ indicates the results are more fair. If $\Delta DP$ is equal to 0, the results satisfy demographic parity and $\hat{Y}$ is independent of $A$. In this section, all the results are the average based on 10 independent replications.

\begin{example}\label{example2}
In this example, we study a special condition in which $Y\indep A$ and we demonstrate that FSRL can achieve both fairness and accuracy under this independence condition.
    We consider the following two cases:
\vspace{-0.5 cm}
 \begin{align*}
 \text{Linear Case}:\   &P(Y=1|X) = \exp\{f(X)\} / \left[1+ \exp\{f(X)\}\right],\\
    &P(A=1|X) = \exp\{g(X)\} / \left[1+ \exp\{g(X)\}\right],
\end{align*}
\text{where } $f(x)=X_1+2X_2-4X_3+1 \text{ and }
g(x) = X_4 + 2X_5 +X_6-1.$

\vspace{-1.4 cm}
\begin{align*}
   \text{Nonlinear Case}:\   &P(Y=1|X) = \exp\{f(X)\} / \left[1+ \exp\{f(X)\}\right],\\
    &P(A=1|X) = \exp\{g(X)\} / \left[1+ \exp\{g(X)\}\right],
\end{align*}
\text{where }  $X\sim N(0,I_{50}), f(X)=X_1+X_2+X_2X_3+\sin(2X_3X_4)+1$ \text{ and } $g(X)=(X_5+X_6)^2 + X_8\cos(X_7^2-1)+\frac{\exp(X_7+2X_8-3)}{2}-2.$
\end{example}

We note that $Y$ and $A$ depend on different entries of $X$, indicating their independence. To learn the representation model, we use a linear function for the linear case and a two-layer Multilayer Perceptron (MLP) model with 32 hidden units for the nonlinear case. The dimension of representation is set to 8 for both cases and we use the logistical regression model for prediction. We vary $\alpha$ from 0.1 to 1 in an incremental step of 0.05 and $\lambda$ is set to 0.001 for both cases.

As shown in Figure \ref{simfig1}, FSRL consistently demonstrates comparable performances to DNN in both linear and nonlinear cases. Specifically, since the target variable $Y$ is independent of sensitive attribute $A$, the representation could achieve both sufficiency and fairness. FSRL maintains a small error rate and  $\Delta  DP$ which only fluctuate within a small range and is not sensitive to $\alpha$.

\begin{figure}[!htbp]
\centering
\includegraphics[width=5.8 in, height=2.8 in]{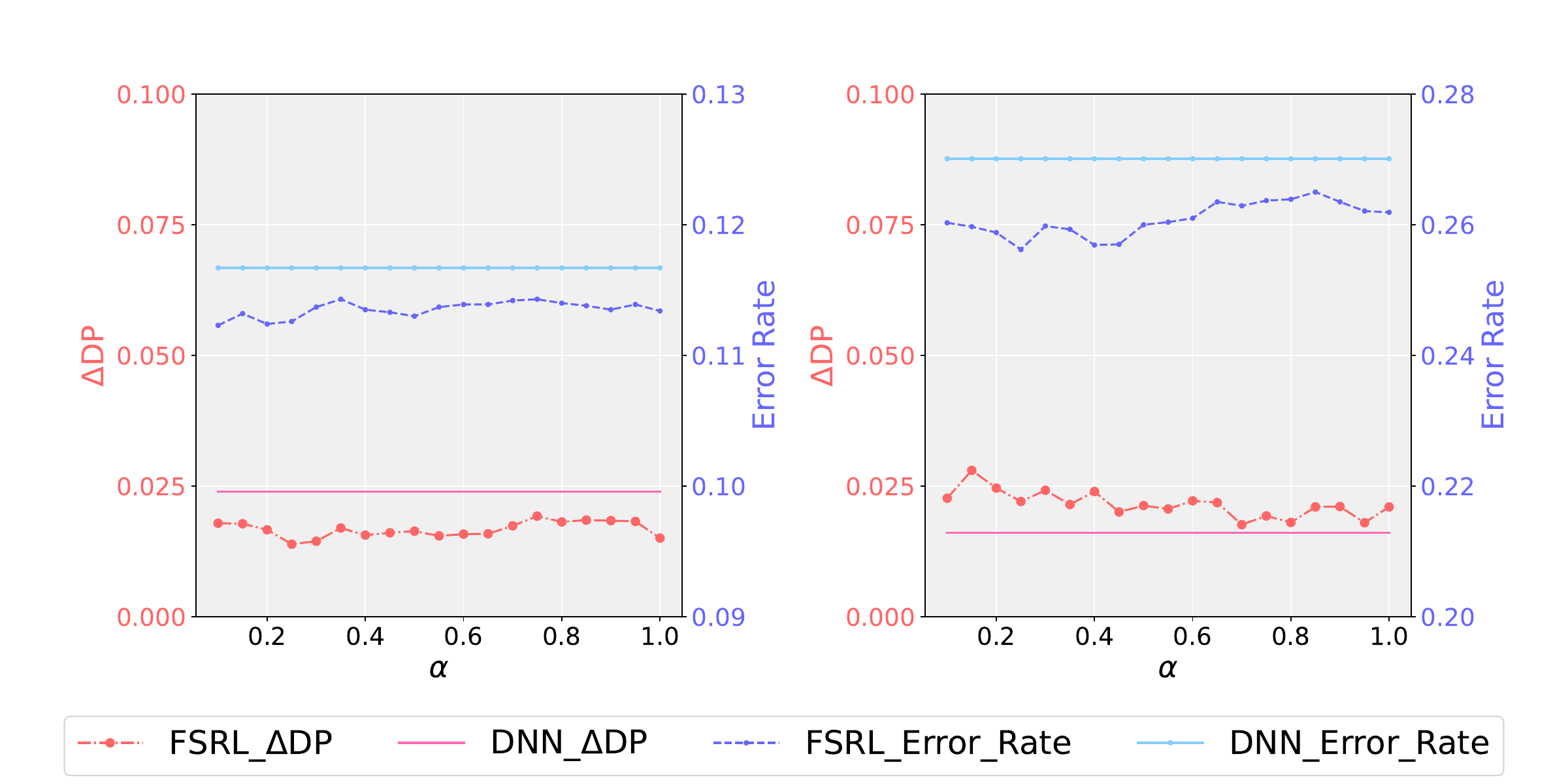}
\caption{Example \ref{example2}: FSRL performs stably when $Y\indep A$.
Left panel: linear case; right panel: nonlinear case.}
\label{simfig1}
\end{figure}

\begin{example}\label{example3}
In this example, we evaluate the performance of FSRL when $Y$ and $A$ are dependent,
which is a  more reasonable scenario in real-world problems.
We demonstrate that FSRL is able to the trade-offs between accuracy and fairness.
We consider the following two cases:

\vspace{-0.5 cm}
\begin{align*}
  \text{Linear Case}:  &P(Y=1|X) = \exp\{f(X)\} / \left[1+ \exp\{f(X)\}\right],\\
    &P(A=1|X) = \exp\{g(X)\} / \left[1+ \exp\{g(X)\}\right],
\end{align*}
 \text{where } $f(X)=X_1+2X_2-4X_3 + X_4 + 2X_5 +1 \text{ and }
g(X) = X_4 + 2X_5 +X_6-1.$

\vspace{-1.4 cm}
\begin{align*}
\text{Nonlinear Case}:
    &P(Y=1|X) = \exp\{f(X)\} / \left[1+ \exp\{f(X)\}\right],\\
    &P(A=1|X) = \exp\{g(X)\} / \left[1+ \exp\{g(X)\}\right],
\end{align*}
 \text{where } $X\sim N(0,I_{50}), f(X)=X_1+X_2+X_2X_3+\sin(2X_3X_4) + (X_5+X_6)^2 + X_8\cos(X_7^2-1)-1 \text{ and } g(X)=(X_5+X_6)^2 + X_8\cos(X_7^2-1)+\frac{\exp(X_7+2X_8-3)}{2}-2.$
\end{example}

In Example \ref{example3}, $f(X)$ and $g(X)$ share some common components and establish inherent connections between $Y$ and $A$. Settings of models and parameters are the same as those in example \ref{example2}.

\begin{figure}[!htbp]
\centering
\includegraphics[width=5.8 in, height=2.8 in]{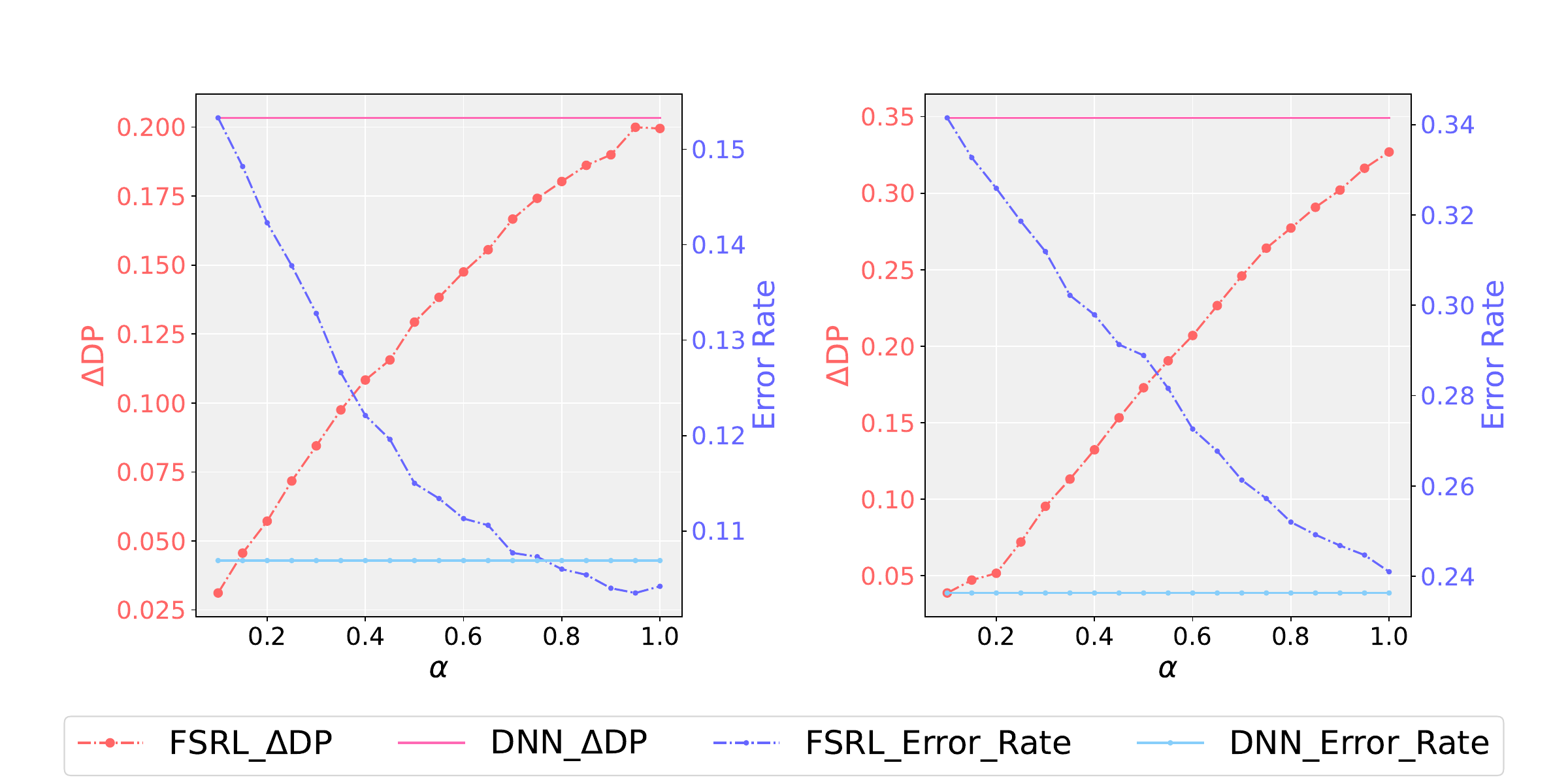}
\caption{Example \ref{example3}:  FSRL achieves different balances between prediction and fairness by varying $\alpha$.Left panel: linear case; right panel: nonlinear case.}
\label{simfig2}
\end{figure}

In Figure \ref{simfig2}, we varied $\alpha$ from 0.1 to 1 in an incremental step of 0.05 and both criteria show obvious responses to changes in $\alpha$. The $\Delta DP$ of FSRL demonstrates a clear growth trend with the increase of $\alpha$ and is almost always smaller than DNN. The error rate is decreasing, indicating an improvement in accuracy. When $\alpha$ is equal to 1, FSRL demonstrates a similar performance as DNN, indicating no sacrifice in prediction. With a suitable $\alpha$, FSRL could achieve the desired balance with a low $\Delta DP$ and error rate.

\begin{figure}[H]
    \centering
    \includegraphics[width=5.8 in, height=2.8 in]{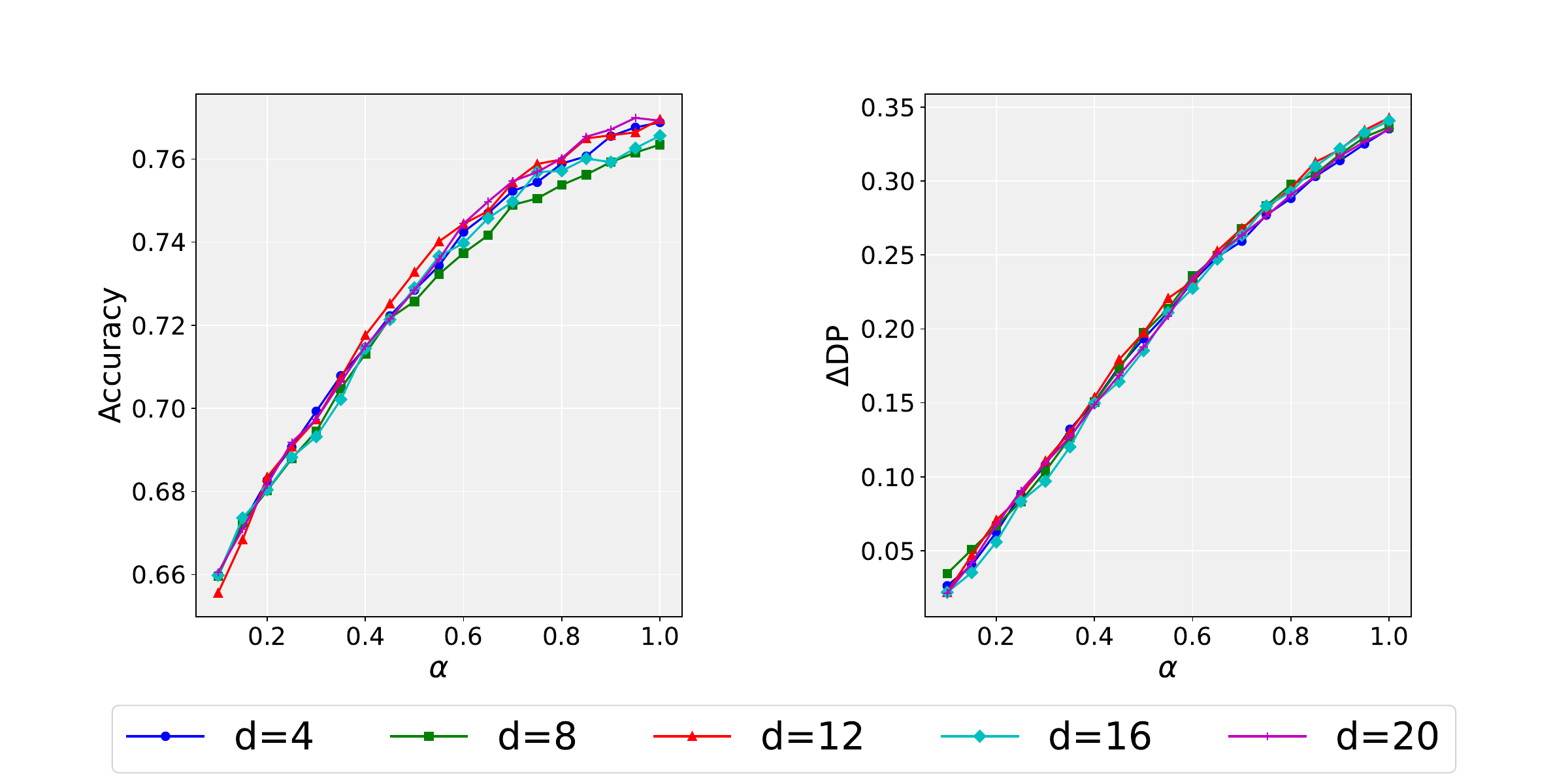}
    \caption{Example \ref{example3}:  performance of FSRL under various dimensions of the  representation. Left panel: accuracy versus $\alpha$ .
Right panel: demographic parity versus $\alpha$.}
    \label{simfig5}
\end{figure}

To further evaluate the effectiveness of FSRL, we analyze its performance across varying dimensions of the representation in the nonlinear case of Example \ref{example3}. As illustrated in Figure \ref{simfig5}, FSRL demonstrates stable performance across different values of $\alpha$ when the representation dimension changes. Notably, even when the dimension is reduced to 4, FSRL effectively identifies essential information and balance accuracy and fairness smoothly through varying $\alpha$.

 \begin{figure}[!htbp]
\centering
\includegraphics[width=5.8 in, height=2.8 in]{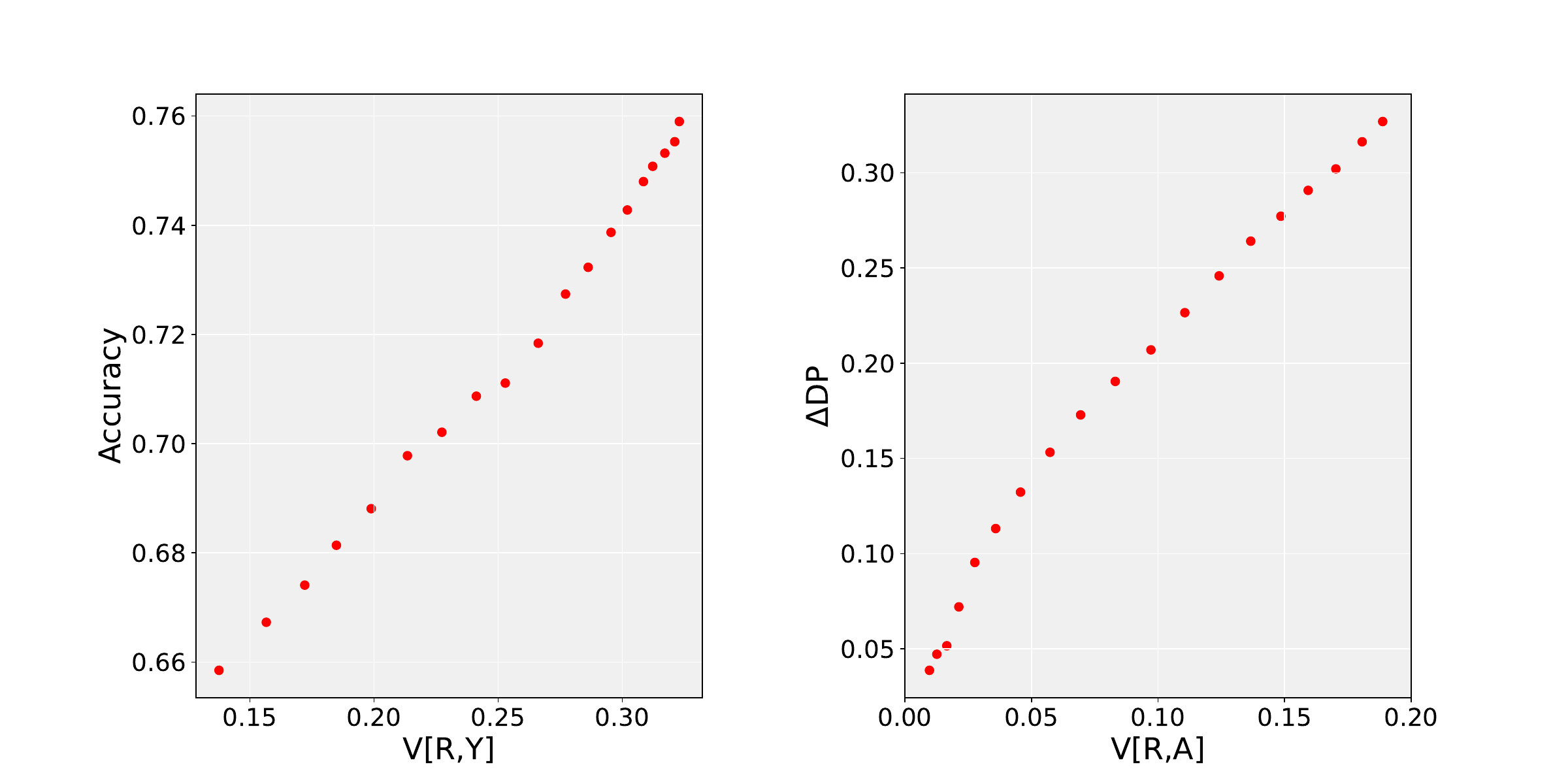}
\caption{Example \ref{example3}: Left panel: relationship between accuracy and distance covariance.
Right panel: relationship between demographic parity and distance covariance.}
\label{simfig3}
\end{figure}

 Figure \ref{simfig3} shows there are strong positive correlations between accuracy and $\mathbb{V}[R,Y]$, as well as between $\Delta DP$ and $\mathbb{V}[R,A]$. The relation between accuracy and $\mathbb{V}[R,Y]$ demonstrates the feasibility of learning a representation with sufficient predictive power by maximizing $\mathbb{V}[R,Y]$. Besides, the relationship between $\Delta DP$ and $\mathbb{V}[R,A]$ suggests that adopting distance covariance as a constraint in fair representation learning is appropriate. Combining the observations in Figure \ref{simfig2} and \ref{simfig3}, FSRL could directly control the balance in representation level, ultimately achieving the trade-offs between accuracy and fairness in final predictions.
For $\alpha \in (0,1]$, FSRL controls and varies the trade-off in a wide range conveniently.

\begin{example}\label{spexample1}
In this example, we illustrate that splitting data to training set for representation learning and a set for prediction function estimation does not significantly affect the experimental results.
We consider the following two cases:
\vspace{-0.5 cm}

    \begin{align*}
    \text{Independent Case:}
    &P(Y=1|X) = \exp\{f(X)\} / \left[1+ \exp\{f(X)\}\right],\\
    &P(A=1|X) = \exp\{g(X)\} / \left[1+ \exp\{g(X)\}\right],
\end{align*}
where $X\sim N(0,I_{50})$, $f(X)=X_1+X_2+X_2X_3+\sin(2X_3X_4)+1$ and $g(X)=(X_5+X_6)^2 + X_8\cos(X_7^2-1)+\frac{\exp(X_7+2X_8-3)}{2}-2$.

\vspace{-1.4 cm}
\begin{align*}
\text{Dependent Case:}
    &P(Y=1|X) = \exp\{f(X)\} / \left[1+ \exp\{f(X)\}\right],\\
    &P(A=1|X) = \exp\{g(X)\} / \left[1+ \exp\{g(X)\}\right],
\end{align*}
where $X\sim N(0,I_{50})$, $f(X)=X_1+X_2+X_2X_3+\sin(2X_3X_4) + (X_5+X_6)^2 + X_8\cos(X_7^2-1)-1$ and $g(X)=(X_5+X_6)^2 + X_8\cos(X_7^2-1)+\frac{\exp(X_7+2X_8-3)}{2}-2$.
\end{example}
In Example \ref{spexample1}, we note that both $Y$ and $A$ have non-linear relationship with covariates $X$. Especially, in the dependent case, $f(X)$ and $g(X)$ share some common components, which builds connections between $Y$ and $A$. We generate a total of 10,000 training samples, 1,000 validation samples, and 1,000 testing samples independently in each case. FSRL use all the training data without separation to estimate both the representation function and prediction function. We randomly split the training samples for learning representation and prediction function with the ratio equal to 4:1, indicating that 8000 samples are used for representation learning and 2000 samples are used for the prediction function. We denote our approach with separation training data as FSRL\_Split.

\begin{figure}[H]
\centering
\includegraphics[width=5.8 in, height=2.8 in]{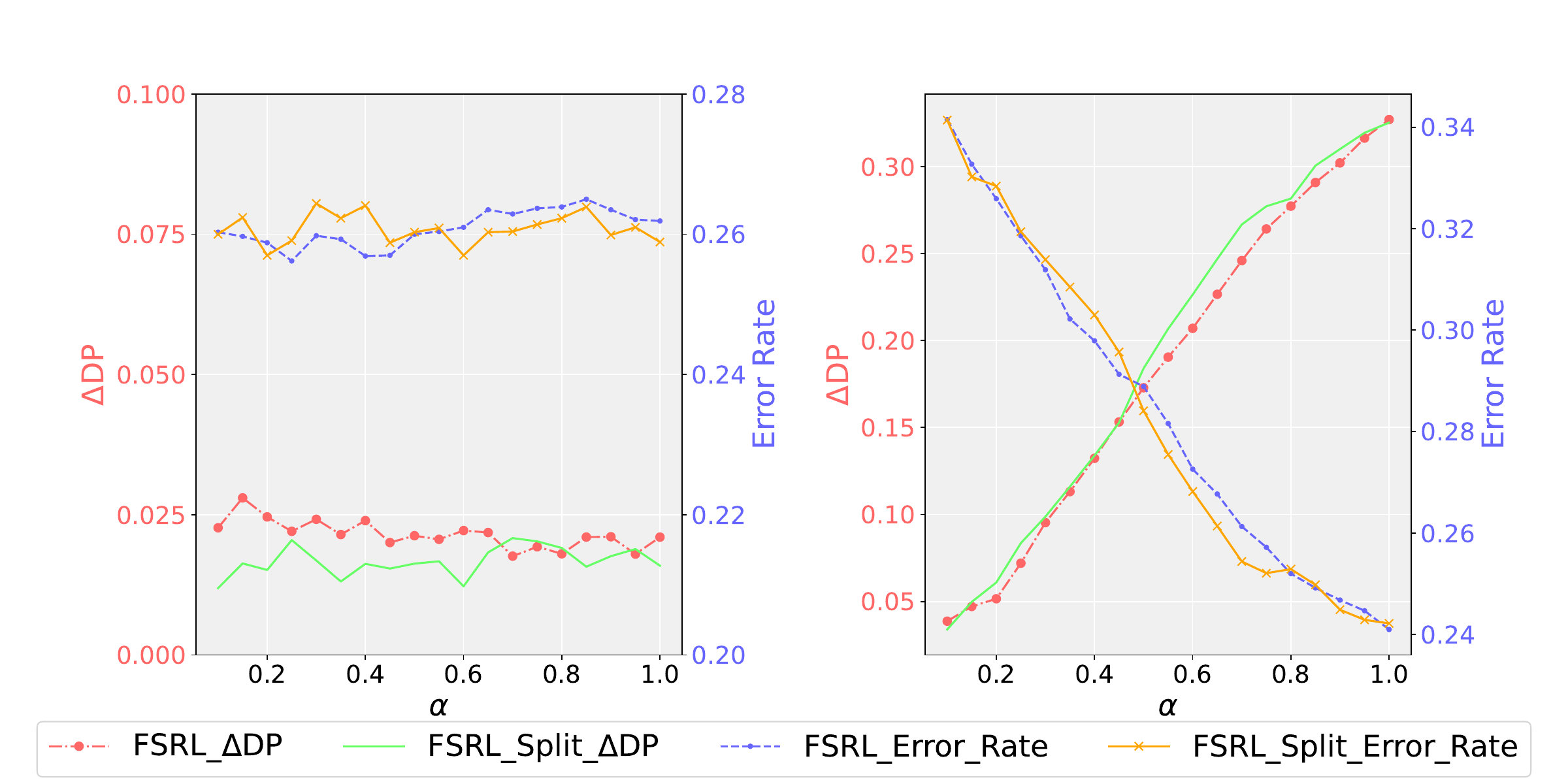}
\caption{Example \ref{spexample1}: FSRL and FSRL\_Split show similar results in both independent and dependent cases. Left panel: independent case; right panel: dependent case.}
\label{spsimfig1}
\end{figure}

As shown in Figure \ref{spsimfig1}, FSRL and FSRL\_Split show comparable performance in both independent and dependent cases. In particular, both methods achieves different balances between $\Delta DP$ and error rate by varying $\alpha$. These observations suggest that whether the data for representation learning and the prediction function are separated has no significant effect on the results.

\section{Real Data Examples} 
Fair machine learning is crucial for addressing real-world problems. In this section, we demonstrate the broad applicability of FSRL through a series of comprehensive experiments. For all the experiments, we first learn the representation through FSRL and then train the classifier with cross-entropy loss while keeping the representation function frozen.

\subsection{UCI Adult  and Heritage Health datasets}
We apply FSRL to the UCI Adult dataset \citep{misc_adult_2} and the Heritage Health dataset \citep{hhp}. The UCI Adult dataset comprises 48,842 instances with 14 attributes, aiming to predict whether an individual's income exceeds \$50,000, with gender designated as the sensitive variable. The Heritage Health dataset includes information on over 60,000 patients, where the target is a binary label indicating whether the Charlson index is greater than zero. Age is considered the sensitive attribute, and patients are divided into eight groups based on different age ranges. We randomly split these datasets into training, validation, and testing sets with a ratio of 8:1:1.

For representation learning, we use a linear function for the UCI Adult dataset and a two-layer multilayer perceptron (MLP) with 128 hidden units for the Heritage Health dataset. The representation dimension is set to 8 for both datasets, and we apply a logistic regression model as the classifier. For the UCI Adult dataset, the parameter $\alpha$ is varied from 0.1 to 1 in increments of 0.05. For the more compact Heritage Health dataset, $\alpha$ is varied from 0.01 to 0.1 in increments of 0.01, and from 0.1 to 1 in increments of 0.05. The parameter $\lambda$ is set to 0.001 for both datasets. All results are averaged over 10 independent replications. Additional implementation details are provided in Section B of the Supplementary Materials.

For a comprehensive comparison, we consider recent works on fair representation learning: \textbf{LAFTR} \citep{madras2018learning}, \textbf{FCRL} \citep{gupta2021controllable}, and \textbf{Dist-Fair} \citep{guo2022learning}. These methods explored fair representation learning from different perspectives. LAFTR utilized adversarial training to eliminate sensitive information and FCRL made use of contrastive learning methods to estimate the lower bound of mutual information. Both approaches learn the representation first and are evaluated by studying another prediction function with representation. Dist-Fair used the VAE structure and encoded three characters, including the reconstruction ability, prediction power and fairness, within a single objective function. Results of all these methods under differing levels of fairness constraints are presented in Figure \ref{realfig1}. DNN is also utilized as the unfair baseline.

As shown in Figure \ref{realfig1}, FSRL demonstrates superior trade-offs compared to other methods. Specifically, FSRL achieves higher accuracy for the same fairness and lower $\Delta DP$ for the same accuracy. These results suggest that FSRL effectively extracts information and delivers enhanced performance. Furthermore, FSRL exhibits the ability to achieve a range of balances between accuracy and $\Delta DP$ across a wide range. By placing greater emphasis on fairness, FSRL attains a smaller $\Delta DP$ while sacrificing the least accuracy compared to all other methods. Compared to FSRL, the performance of Dist-Fair is insensitive to the strength of fairness constraints across both datasets, which may be misled by the emphasis on the reconstruction ability. Due to the inherent instability of adversarial training, LAFTR exhibits a similar pattern to Dist-Fair on the UCI Adult dataset and shows a significant decline in predictive performance on the Heritage Health dataset. Although FCRL shows some trade-offs on both datasets, the relationship between accuracy and $\Delta DP$ remains unclear. Many results exhibit the same $\Delta DP$ but differing accuracies, suggesting that the algorithm’s convergence is ambiguous and its reliability is compromised. Among all the methods, FSRL stands out as the most reliable method, consistently achieving superior trade-offs between accuracy and $\Delta DP$ across all datasets.

\begin{figure}[H]
\centering
\includegraphics[width=5.8 in, height=2.8 in]{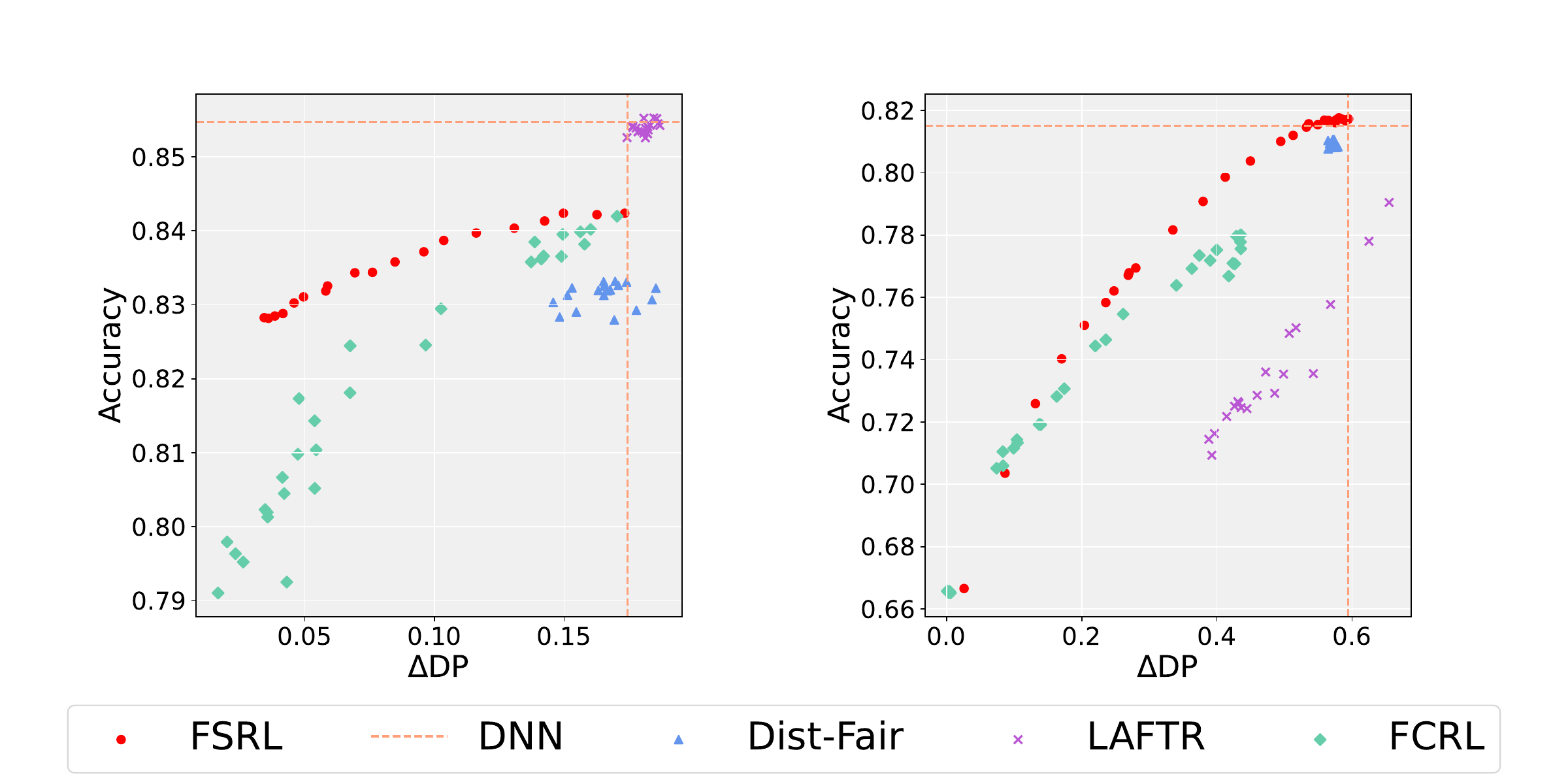}
\caption{By varying the coefficient of fair constraints, all methods achieve different trade-offs between Accuracy and $\Delta DP$. Left panel: UCI Adult dataset; right panel: Heritage Health dataset.}
\label{realfig1}
\end{figure}

We show the convenience of controlling the trade-off on real-world datasets through $\alpha$ in Figure \ref{realfig2}. There is a clear trend that $\Delta DP$ gradually increases and the error rate gradually decreases with the improvement of $\alpha$. As $\alpha$ decreases, the representation places more emphasis on sensitive information, achieving a smaller $\Delta DP$. Combined with observations in Figure \ref{realfig1}, with an appropriate choice of $\alpha$, FSRL could achieve optimal performance while satisfying the same fairness constraints.

\begin{figure}[H]
\centering
\includegraphics[width=5.8 in, height=2.8 in]{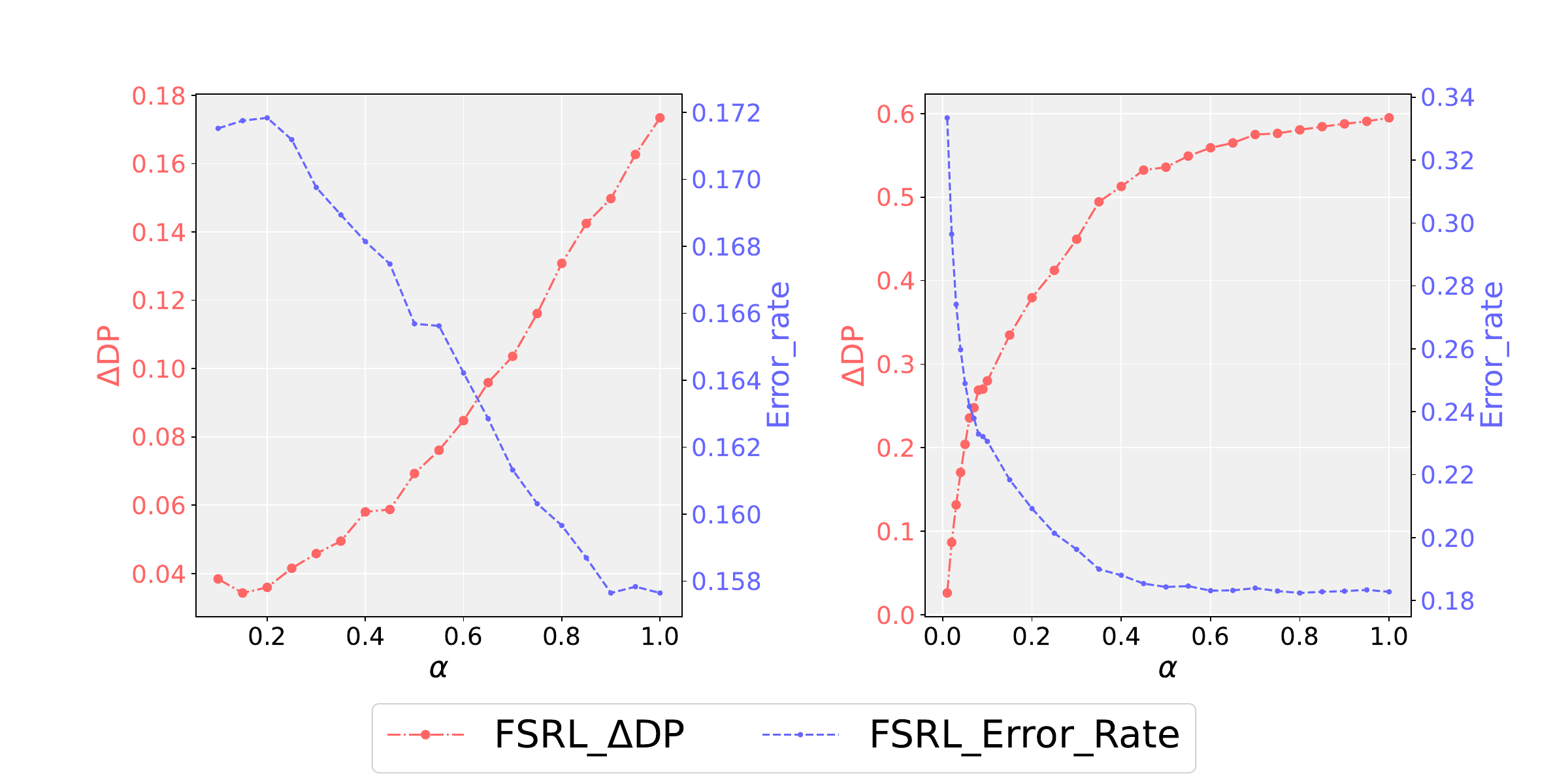}
\caption{On both the UCI Adult dataset and the Heritage Health dataset, FSRL achieves different balances between prediction and fairness by varying $\alpha$. Left panel: UCI Adult dataset; right panel: Heritage Health dataset.}
\label{realfig2}
\end{figure}

\subsection{Bias-in-Bios data}
The Bias-in-Bios dataset is a collection of online professional biographies designed to study and mitigate bias in natural language processing models \citep{dearteaga2019bias}. It comprises approximately 393,000 biographies of individuals from 28 professions, annotated with gender labels. We utilize this dataset to construct a classification model that predicts an individual's profession based on their biography. Gender is set as the sensitive attribute, and all words related to gender information are removed from the biographies.

In our analysis, we employ a pre-trained BERT model \citep{devlin2018bert} as the encoder, using its CLS embedding for the downstream task. Following \citet{de2019bias}, the dataset is divided into three parts: 255,710 instances for training (65\%), 39,369 for validation (10\%), and 98,344 for testing (25\%). Without fine-tuning the encoder, we apply FSRL to adjust the unfair BERT CLS embeddings. For representation learning, we use a two-layer multilayer perceptron (MLP) with a hidden dimension of 768, and a logistic regression model is employed for classification.

We compare FSRL with the following recently developed text debiasing representation learning techniques: \text{INLP} \citep{ravfogel2020null}, \text{AdvEns} \citep{han2021diverse}, and \text{SUP} \citep{shi2024debiasing}. INLP iteratively projects the representation into the null space of the sensitive variable, while SUP learns a fair representation by subtracting the linear sufficient representation of the sensitive variable. Similar to LAFTR, AdvEns employs adversarial training using an ensemble-based adversary.

Below, we use \text{DNN} to denote the unfair model trained with cross-entropy loss. For a more comprehensive comparison, we considered the widely used metric \text{GAP}  \citep{de2019bias,ravfogel2020null}, defined as
\begin{align*}
    \text{GAP} = \sqrt{\frac{1}{|\mathcal{M}|}\sum_{y\in \mathcal{M}}(\text{GAP}_{A,y}^{\text{TPR}})^2},
\end{align*}
where $\mathcal{M}$ is the collection of all possible values of $Y$ and  is the difference of true positive rate (TPR) between binary sensitive variable $A$ in class $y$. $\text{GAP}_{A,y}^{\text{TPR}}$ is strongly correlated with another fairness metric, Equalized Odds (EO), requiring the condition independence between prediction $\hat{Y}$ and sensitive variable $A$ given target $Y$. EO can be satisfied by ensuring the conditional independence between representation $R(X)$ and $A$ and we leave the investigation of their conditional independence for future work.
More details of the experiment are presented in Section B of the Supplementary Materials.

\begin{table}[H]
    \centering
    \resizebox{0.6\textwidth}{!}{
    \begin{tabular}{cc cc}
    \hline
     Dataset & Method     & Accuracy$\uparrow$ & GAP$\downarrow$   \\
         \hline
\multirow{5}{*}{Bias-in-Bios} & DNN    & \textbf{81.10}$\pm$0.08  & 16.30$\pm$0.37\\
                & INLP  & 74.33$\pm$0.45  & 7.86$\pm$0.42 \\
                & AdvEns   & 60.15$\pm$0.58  & \textbf{6.10}$\pm$0.76 \\
                & SUP   & 79.26 $\pm$0.08  & 15.66$\pm$1.06\\
                & FSRL  & 80.38$\pm$0.06  & 12.74$\pm$0.17   \\
         \hline
    \end{tabular}
    }
    \caption{The performance of different methods on Bias-in-Bios dataset.}
    \label{tabel1}
\end{table}

We set $\alpha$ to 0.5 and present the results that are averaged over 5 runs in Table \ref{tabel1}. Emphasizing both sufficiency and fairness, FSRL achieves both higher accuracy and smaller $\Delta DP$ than SUP. Moreover, employing the general nonlinear representation function enables FSRL to capture informative features that are independent of $A$ more effectively.

\begin{figure}[H]
    \centering
    \includegraphics[width=5.8 in, height=2.8 in]{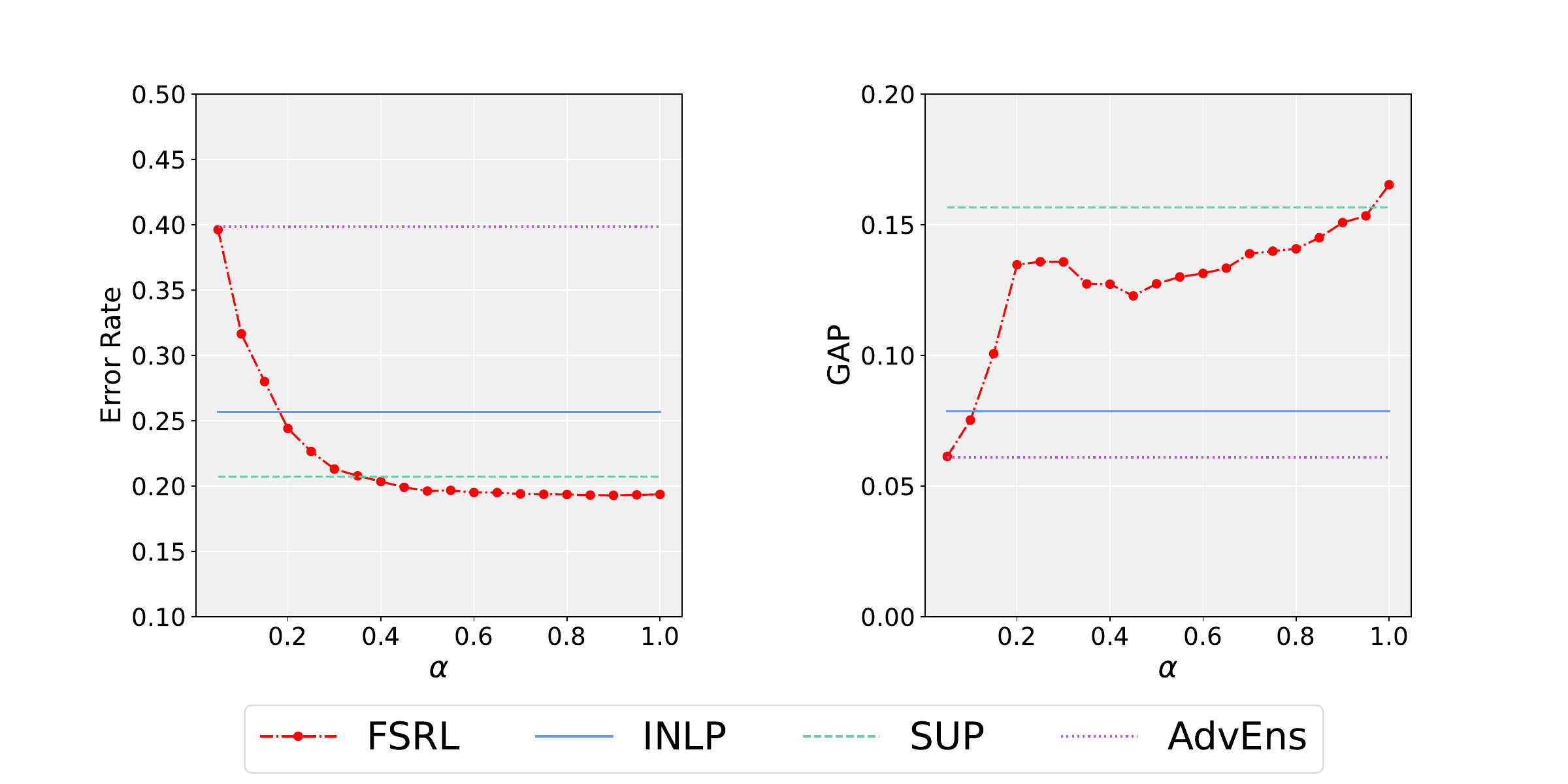}
    \caption{The performance of FSRL on Bias-in-Bios dataset with varying $\alpha$ values.}
    \label{realfig4}
\end{figure}

To better compare FSRL with AdvEns and INLP, we illustrate the trends of Error Rate and GAP in Figure \ref{realfig4} as $\alpha$ changes from 0.05 to 1 with step 0.05. AdvEns achieves the best fairness but it sacrifices too much predictive performance, leading to undesired overall results. INLP achieves a reasonable balance between accuracy and fairness. However, INLP requires projecting the embedding to the null space of predicting $A$, lacking the flexibility to adjust the balance. Compared to INLP, FSRL can effectively control the balance between predictive performance and fairness by adjusting $\alpha$. These observations about effectively extracting information and mitigating biases suggest that FSRL could achieve great trade-offs in text data and have huge potential in multiple tasks.

\begin{figure}[H]
    \centering
    \includegraphics[width=6.4 in, height=4.2 in]{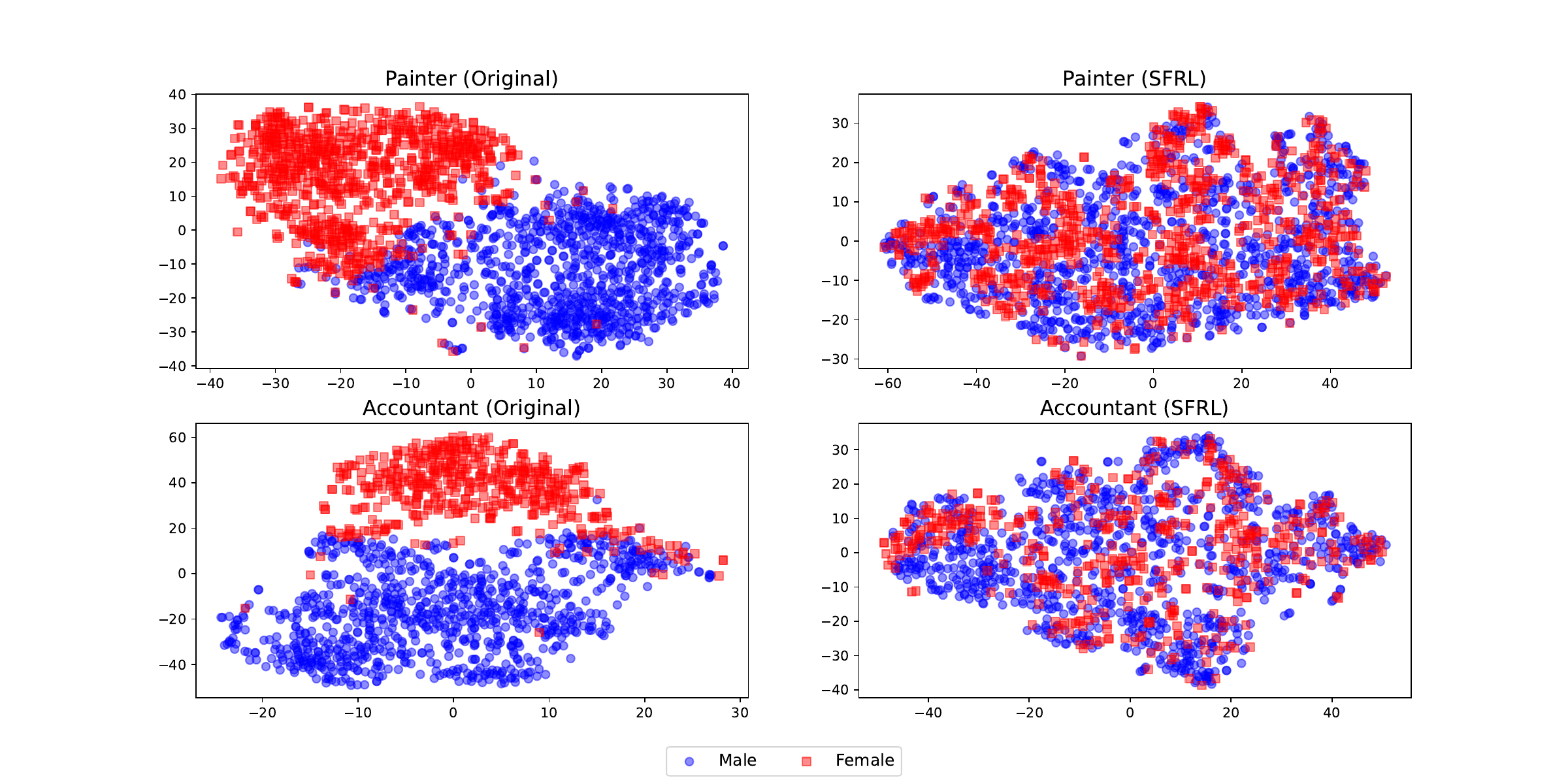}
    \caption{t-SNE projection of the embeddings. The original embeddings:
Painter (top left) and Accountant (bottom left). The FSRL representations: Painter (top right) and Accountant (bottom right).}
    \label{realfig6}
\end{figure}

To further assess the effectiveness of FSRL, we present t-SNE visualizations of the original BERT embeddings and the representations learned using FSRL with
$\alpha=0.5$  in Figure \ref{realfig6}. In the original BERT embeddings, both the Painter and Accountant professions are divided into two distinct groups based on gender, indicating that a significant amount of sensitive information is embedded. After applying FSRL, these subgroups become mixed, making it difficult to distinguish gender from the representation. This shows that FSRL effectively protects sensitive information within the data.

\begin{figure}[H]
    \centering
    \includegraphics[width=6.4 in, height=4.2 in]{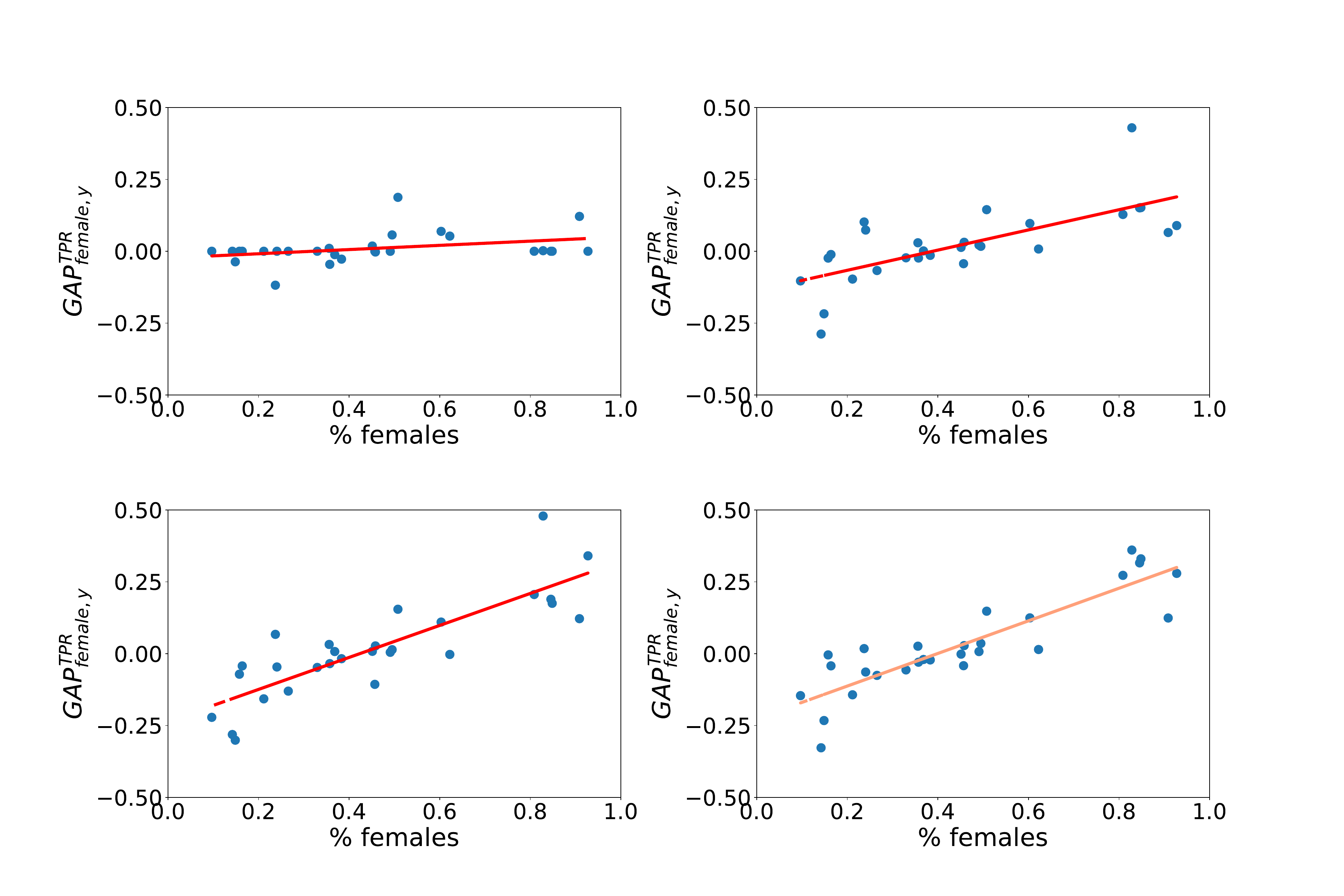}
    \caption{The correlation between $\text{GAP}^{\text{TPR}}_{
\text{female},y}$ and the relative proportion of females in profession $y$ in results from FSRL with different $\alpha$ values and unfair baseline DNN. Each point represents a profession.}
    \label{realfig5}
\end{figure}

We also illustrate the correlation between  $\text{GAP}^{\text{TPR}}_{\text{female},y}$
 and the relative proportion of women in profession $y$ in Figure \ref{realfig5}, where each point represents a profession. The first three subplots correspond to
$\alpha$ values of 0.05, 0.5, and 1.0, respectively, while the last subplot shows the results of a DNN without any fairness constraints. As
$\alpha$ increases, there is a clear trend of increasing correlation. When
$\alpha=0.5,$ FSRL eliminates most of the sensitive information, resulting in a correlation close to zero. As $\alpha$ increases to 1, the representation is only required to be sufficient, exhibiting a pattern similar to that of the DNN. These observations demonstrate that FSRL can effectively remove sensitive information from the pre-trained BERT model and achieve a balanced representation that maintains both prediction accuracy and fairness.

\section{Conclusion}
In this paper, we propose FSRL, a fair representation learning method based on statistical independence and sufficiency. We  characterize the properties of fair sufficient representation learning and introduce a direct approach to estimate the desired representation. Our work enhances the understanding of fair representation learning and provide insights into the effectiveness of balancing sufficiency and fairness in representation. We consider the distance covariance as the information metric and use a convex combination to build the objective function. Experiment results demonstrate that FSRL achieves improved trade-offs in datasets with various structures. They also indicate that FSRL could achieve the desired balance between prediction and fairness under different fair constraints and demonstrate the great potential for its applications.

In this paper, we focus on demographic parity and leave the discussion of other notions of fairness such as equalized odds for future exploration. We also emphasize that it is possible to consider other information measures in application, such as the mutual information and Hilbert-Schmidt independence criterion.
Future research could explore applying fair representation learning to fine-tune large language models to ensure their outcomes are not influenced by sensitive information.

\bigskip
\bigskip
\appendix

\renewcommand{\thetable}{A\arabic{table}}
\renewcommand{\thefigure}{A\arabic{figure}}
\renewcommand{\theequation}{\thesection.\arabic{equation}}
\renewcommand{\thetheorem}{\thesection.\arabic{theorem}}
\renewcommand{\theassumption}{\thesection.\arabic{assumption}}

\bigskip\noindent
\textbf{\LARGE Appendix}

\medskip\noindent
This appendix contains the details in implementations,  proofs of the results in  the paper.

\section{Technical Results}
\subsection{Proof of Theorem 1}
\begin{lemma}[Theorem 1 in \citet{huang2024deep}]
\label{lemma111}
    For random variables $(X,Y)\in \mathbb{R}^p\times \mathbb{R}^q$, we have $R^*(X)\in \argmin_{R(X)\sim N(0,I_d)} -\mathbb{V}[R(X),Y]$ provided
    \begin{align*}
        X\indep Y |R^*(X)
    \end{align*}
    holds.
\end{lemma}
With the help of Lemma \ref{lemma111}, we give the proof of Theorem 1.
\begin{proof}
    As $Y\indep A$, we have $R_0$ that satisfies
    \begin{align*}
        &X\indep Y|R_0(X)\\
        &R_0(X)\indep A
    \end{align*}
    and for any $R(X)\sim N(0,I_d)$, $\alpha\in(0,1]$ and $\lambda>0$, we have
    \begin{align*}
        \mathcal{L}(R_0)-\mathcal{L}(R)&=\alpha\big(-\mathbb{V}[R_0,Y]+\mathbb{V}[R,Y]\big) + (1-\alpha)\big(\mathbb{V}[R_0,A]-\mathbb{V}[R,A]\big)\\
        &+\lambda \big(\mathbb{D}[R_0,\gamma_d]-\mathbb{D}[R,\gamma_d]\big)\\
        &=\alpha\big(-\mathbb{V}[R_0,Y]+\mathbb{V}[R,Y]\big) - (1-\alpha)\mathbb{V}[R,A]\\
        &\leq \alpha\big(-\mathbb{V}[R_0,Y]+\mathbb{V}[R,Y]\big)\\
        &\leq 0.
    \end{align*}
The second equality holds since $R_0\indep A$ and both $R_0$ and $R$ follow $N(0,I_d)$. The third inequality holds since $\alpha\in(0,1]$ and $\mathbb{V}[R,A]\geq 0$. The last inequality holds since $R_0$ is a sufficient representation for $(X,Y)$.
\end{proof}
\subsection{Proof of Theorem 2}
In this subsection, we analyze the bound of excess risk and build theoretical foundations for the learned representation. Before presenting the detailed proofs, we first recall some assumptions of the neural network and random variables. Denote $R_0(X)=\argmin \mathcal{L}(R)$.

\begin{assumption}\label{Ass1}
      Each component of $R_0(X)$ is $\beta$-H{\"o}lder smooth on $[0,1]^p$ with constant $\mathcal{B_0}$, that is $R_{0,i}\in \mathcal{H}^\beta([-0,1]^p,B_0), i=1,...,d$.
 \end{assumption}

 \begin{assumption}\label{Ass2}
      The support of $X$ is contained in a compact set $[0,1]^p$. $Y$ and $A$ are bounded almost surely: $||Y||\leq B_1$ , $||A||\leq B_2$, a.s..
 \end{assumption}

 \begin{assumption}\label{Ass3}
     Parameters of the representation network $\mathbf{R}\equiv \mathbf{R}_{\mathcal{D,W,S,B}}$ : depth $\mathcal{D}= \mathcal{O}(\log n\log_2(\log n))$, width $\mathcal{W} = \mathcal{O}(n^{\frac{p}{2(2\beta+p)}}\log_2(n^{\frac{p}{2(2\beta+p)}}/\log n)/\log n)$, size $\mathcal{S} = \\ \mathcal{O}(dn^{\frac{p}{2\beta+p}}/\log^4(npd))$, boundary $\sup_{X\in[0,1]^p}||R(X)||_2\leq\mathcal{B}$.
 \end{assumption}

 Denote $\hat{R}(X)\in\argmin_{R\in\mathbf{R}_{\mathcal{D,W,S,B}}}\mathcal{L}_n(R)$ and for any $\tilde{R}(X) \in \mathbf{R}_{\mathcal{D,W,S,B}}$, we have
\begin{align*}
    \mathcal{L}(\hat{R})-\mathcal{L}(R_0)&=\mathcal{L}(\hat{R})-\mathcal{L}_n(\hat{R}) +\mathcal{L}_n(\hat{R})-\mathcal{L}_n(\tilde{R})+\mathcal{L}_n(\tilde{R}) -\mathcal{L}(\tilde{R}) + \mathcal{L}(\tilde{R})-\mathcal{L}(R_0)\\
    &\leq 2\sup_{R\in\mathbf{R}_{\mathcal{D,W,S,B}}}|\mathcal{L}(R)-\mathcal{L}_n(R)| + \inf_{R\in\mathbf{R}_{\mathcal{D,W,S,B}}}|\mathcal{L}(R)-\mathcal{L}(T_{B_n}R_0)|\\
 \end{align*}
 where we use the definition of $\hat{R}$ and the feasibility of $\tilde{R}$. We analyze the excess risk
 \begin{align*}
     \mathbb{E}_{(X_i,Y_i,A_i)_{i=1}^n}[\mathcal{L}(\hat{R})-\mathcal{L}(R_0)]
 \end{align*}
 by providing bound for these two error terms:
 \begin{itemize}
     \item the \textbf{approximation error}: $\inf_{\mathcal{D,W,S,B}}|\mathcal{L}(R)-\mathcal{L}(T_{B_n}R_0)|$;
     \item  the \textbf{statistical error}: $\sup_{\mathcal{D,W,S,B}}|\mathcal{L}(R)-\mathcal{L}_n(R)|$.
 \end{itemize}
 Without loss of generality, we set $\alpha$ to $\frac{1}{2}$  and $\lambda$ to 1 in the following analysis.

\subsubsection{The Approximation Error}
\begin{lemma}[Corollary 3.1 in \citet{jiao2023deep}]\label{lemma 4}
    If $f \in \mathcal{H}^\beta([-0,1]^p,B_0)$ with $\beta=s+r$, $s\in \mathbb{N}_0$ and $r\in (0,1]$. For any $M,N\in \mathbb{N}^{+}$, there exists a function $\phi$ implemented by a ReLU network with width $W=38(\lfloor\beta\rfloor+1)^23^pp^{\lfloor\beta\rfloor+1}N\lceil\log_2(8N)\rceil$ and depth $D=21(\lfloor \beta\rfloor+1)^2M\lceil\log_2(8M)\rceil$ such that
\begin{align*}
    |f(x)-\phi(x)|\leq 19B_0(\lfloor\beta\rfloor+1)^2p^{\lfloor\beta\rfloor+(\beta\vee1)/2}(NM)^{-2\beta/p}, x\in[0,1]^p.
\end{align*}
\end{lemma}
\begin{lemma}\label{lemma 5}
    Suppose assumption \ref{Ass1}, \ref{Ass2},\ref{Ass3} hold, then \textbf{the Approximation Error}
\begin{align*}
    \inf_{R\in \mathbf{R}_{\mathcal{D,W,S,B}}}|\mathcal{L}(R)-\mathcal{L}(R_0)|&\leq  \mathcal{O}(n^{\frac{-\beta}{2\beta+p}}).
\end{align*}
\end{lemma}
\begin{proof}

Due to the assumption that $R_{0,i} \in \mathcal{H}^\beta([-0,1]^p,B_0) $, then
 $ \forall i$, $ 1\leq i\leq d$, by Lemma \ref{lemma 4}, there exists a ReLU network $\overline{R}_i$ with width $W=38(\lfloor\beta\rfloor+1)^23^pp^{\lfloor\beta\rfloor+1}N\lceil\log_2(8N)\rceil$ and depth $D=21(\lfloor \beta\rfloor+1)^2M\lceil\log_2(8M)\rceil$, such that
\begin{align*}
    |\overline{R}_i(x)-R_{0,i}(x)|\leq 19B_0(\lfloor\beta\rfloor+1)^2p^{\lfloor\beta\rfloor+(\beta\vee1)/2}(NM)^{-2\beta/p}, x\in[0,1]^p.
\end{align*}
Horizontally splicing these networks $\overline{R}_i$ as $\overline{R}$, we have
\begin{align*}
    \mathbb{E}||\overline{R}(x)-R_0(x)||&= \mathbb{E}\left[\sum_{i=1}^d(\overline{R}_i(x)-R_{0,i}(x))^2\right]^{1/2}\\
    &\leq \sqrt{d}\mathbb{E}\max_{1\leq i\leq d, \atop x\in [0,1]^p}|\overline{R}_i(x)-R_{0,i}(x)|\\
    &\leq 19\sqrt{d}B_0(\lfloor\beta\rfloor+1)^2p^{\lfloor\beta\rfloor+(\beta\vee1)/2}(NM)^{-2\beta/p}.
\end{align*}
Since $\mathbb{V}[X,Y]=\mathbb{E}[||X-X^{'}||(||Y-Y^{'}||-2||Y-Y^{''}|| +\mathbb{E}||Y-Y^{'}||)]$, we have
\begin{align}\label{vv}
    |\mathbb{V}[\overline{R},Y]-\mathbb{V}[R_0,Y]|&\leq \mathbb{E}\big| ||\overline{R}-\overline{R}^{'}||-||R_0-R_0^{'}||\big|\big(||Y-Y^{'}||-2||Y-Y^{''}|| +\mathbb{E}||Y-Y^{'}||\big)\nonumber\\
    &\leq 16B_1 \mathbb{E}\big| ||\overline{R}-R_0||\big|.
\end{align}
Due to $\mathbb{D}[X,Y]=2\mathbb{E}||X-Y||-\mathbb{E}||X-X'||-\mathbb{E}||Y-Y'||$, we have
\begin{align}\label{dd}
    |\mathbb{D}[\overline{R},\gamma_d]-\mathbb{D}[R_0,\gamma_d]|\nonumber
    & \leq 2\mathbb{E}\big|||\overline{R}-\overline{R}^{'}||-\mathbb{E}||R_0-R_0^{'}||\big| + \mathbb{E}\big|||\overline{R}-\gamma_d|| - \mathbb{E}||R_0-\gamma_d||\big|\nonumber\\
    &\leq 5 \mathbb{E}||\overline{R}-R_0||.
\end{align}
Inequality \eqref{vv} and \eqref{dd} hold because of the triangle inequality, and combining the two inequalities, we have
\begin{align}
   & \inf_{R\in \mathbf{R}_\mathcal{D,W,S,B}}|\mathcal{L}(R)-\mathcal{L}(R_0)|\nonumber\\
   &\leq |\mathcal{L}(\overline{R})-\mathcal{L}(R_0)| \nonumber\\
    &=\frac{1}{2}\big|\mathbb{V}[\overline{R},Y]-\mathbb{V}[R_0,Y]\big| +\frac{1}{2}\big|\mathbb{V}[\overline{R},A]-\mathbb{V}[R_0,A]\big|+ \big|\mathbb{D}[\overline{R},\gamma_d]-\mathbb{D}[R_0,\gamma_d]\big|\nonumber\\
    &\leq 8(B_1+B_2+5)\mathbb{E}||\overline{R}-R_0||\nonumber\\
    & \leq 152\sqrt{d}B_0(B_1+B_2+5)(\lfloor\beta\rfloor+1)^2p^{\lfloor\beta\rfloor+(\beta\vee1)/2}(NM)^{-2\beta/p} \label{bound1}.
\end{align}
The first inequality is straightforward since $\overline{R}\in\mathbf{R}_\mathcal{D,W,S,B}$ and the second inequality is obtained by inequality \eqref{vv} and \eqref{dd}. Combining the upper bound \eqref{bound1} and network parameters in Assumption \ref{Ass3}, we have the approximation error that satisfies
\begin{align*}
     \inf_{R\in \mathbf{R}_\mathcal{D,W,S,B}}|\mathcal{L}(R)-\mathcal{L}(R_0)|&\leq  152\sqrt{d}B_0(B_1+B_2+5)(\lfloor\beta\rfloor+1)^2p^{\lfloor\beta\rfloor+(\beta\vee1)/2}n^{\frac{-\beta}{2\beta+p}}\\
     &=\mathcal{O}(n^{\frac{-\beta}{2\beta+p}}).
\end{align*}
\end{proof}

\subsubsection{The Statistical Error}
\begin{lemma}\label{lemma 6}
    If $\xi_i, i=1,...,m $ are m finite linear combinations of Rademacher Variables $\epsilon_j,j=1,...,J$. Then
\begin{align*}
    \mathbb{E}_{\epsilon_j,j=1,...,J}\max_{1\leq i \leq m}| \xi_i|\leq C_1(\log m)^{1/2}\max_{1\leq i\leq m}(\mathbb{E}\xi^2_i)^{1/2}.
\end{align*}
\end{lemma}
\begin{lemma}\label{lemma 7}
    Suppose assumption \ref{Ass1} and \ref{Ass3} hold, given n i.i.d samples $\{(x_i,y_i)\}_{i=1}^n$, we have
\begin{align*}
    \mathbb{E}\sup_{R\in \mathbf{R}_{\mathcal{D,W,S,B}}} \big| \mathbb{V}_n[R,Y]-\mathbb{V}[R,Y]\big| \leq  C_2/n + C_1C_2C_3C_4n^{\frac{-\beta}{2\beta+p}}.
\end{align*}
\end{lemma}
\begin{proof}
Define the kernel function $\tilde{h}_R((x_{i_1},y_{i_1}),\cdots,(x_{i_4},y_{i_4}))$ as:
\begin{align*}
    \tilde{h}_R((x_{i_1},y_{i_1}),\cdots,(x_{i_4},y_{i_4}))=h((R(x_{i_1}),y_{i_1}),\cdots,(R(x_{i_4}),y_{i_4}))-\mathbb{V}[R,Y],
\end{align*}
where the kernel function $h((u_{i_1},v_{i_1}),\cdots,(u_{i_4},v_{i_4}))$ is defined as
\begin{align*}
     h((u_{i_1},v_{i_1}),\cdots,(u_{i_4},v_{i_4})) &=\frac{1}{4}\sum\limits_{1\leq i,j \leq 4 \atop i\neq j}\vert\vert u_i-u_j\vert\vert\vert\vert v_i-v_j\vert\vert \nonumber\\
     &+\frac{1}{24}\sum\limits_{1\leq i,j \leq 4 \atop i\neq j}\vert\vert u_i-u_j\vert\vert\sum\limits_{1\leq i,j \leq 4 \atop i\neq j}\vert\vert v_i-v_j\vert\vert \nonumber\\
     &-\frac{1}{4}\sum\limits_{i=1}^4\left(\sum\limits_{1\leq i,j \leq 4 \atop i\neq j}\vert\vert u_i-u_j\vert\vert\sum\limits_{1\leq i,j \leq 4 \atop i\neq j}\vert\vert v_i-v_j\vert\vert\right).
 \end{align*}
Then, $\mathbb{V}_n[R,Y]-\mathbb{V}[R,Y]$ can be represented as a centered $U$-statistics
\begin{align*}
    \mathbb{V}_n[R,Y]-\mathbb{V}[R,Y]=\frac{1}{C_n^4}\sum\limits_{1\leq i_1 < i_2 < i_3 < i_4\leq n}\tilde{h}_R((x_{i_1},y_{i_1}),\cdots,(x_{i_4},y_{i_4})).
\end{align*}
By the symmetrization randomization Theorem 3.5.3 in \cite{de2012decoupling},
\begin{align}\label{hr1}
    &\mathbb{E}\sup_{R\in \mathbf{R}_\mathcal{D,W,S,B}}| \mathbb{V}_n[R,Y]-\mathbb{V}[R,Y]| \nonumber\\
    &\leq C_2\mathbb{E}\sup_{R\in \mathbf{R}_\mathcal{D,W,S,B}}\left|\frac{1}{C_n^4}\sum\limits_{1\leq i_1 < i_2 < i_3 < i_4\leq n}\epsilon_{i_1}\tilde{h}_R((x_{i_1},y_{i_1}),\cdots,(x_{i_4},y_{i_4}))\right|,
\end{align}
where $\epsilon_{i_1},i_1=1,...,n$ are i.i.d Rademacher variables that are independent with samples.\\

Based on Assumptions \ref{Ass2} and \ref{Ass3}, the kernel $\tilde{h}_R$ is also bounded. $\forall {R \in \mathbf{R}_\mathcal{D,W,S,B}},$
we define a random empirical measure for $R$ and $\tilde{R}$,
$$
e_{n,1}(R, \tilde{R})= \bE_{\epsilon_{i_1}, i_1 = 1,...,n}\left|\frac{1}{C_n^4} \sum_{1\leq i_1<  i_2<i_3<i_4\leq n} \epsilon_{i_{1}}(\tilde{h}_{R}-\tilde{h}_{\tilde{R}})((x_{i_1},y_{i_1}),\cdots,(x_{i_4},y_{i_4}))\right|.$$

Condition on $\{(x_i,y_i)\}_{i=1,\ldots, n}$, let $\zeta(\mathbf{R}, e_{n,1}, \delta)$ be  the covering number of the neural network class $\mathbf{R}\equiv \mathbf{R}_\mathcal{D,W,S,B}$ with respect  to the  empirical  distance  $e_{n,1}$ at scale of $\delta>0$.
With Lemma \ref{lemma 6} and the above analysis, we build the connection between the expectation \eqref{hr1} and the covering set $\mathbf{R}_\delta$. Condition on $\{(x_i,y_i)\}_{i=1}^n$,
\begin{align*}
    &\mathbb{E}_{\epsilon_{i_1}}\sup_{R\in \mathbf{R}_\mathcal{D,W,S,B}}\left|\frac{1}{C_n^4}\sum\limits_{1\leq i_1 < i_2 < i_3 < i_4\leq n}\epsilon_{i_1}\tilde{h}_R((x_{i_1},y_{i_1}),\cdots,(x_{i_4},y_{i_4}))\right|\\
    &\leq \delta +E_{\epsilon_{i_1}}\sup_{R\in R_\delta}\left|\frac{1}{C_n^4}\sum\limits_{1\leq i_1 < i_2 < i_3 < i_4\leq n}\epsilon_{i_1}\tilde{h}_R((x_{i_1},y_{i_1}),\cdots,(x_{i_4},y_{i_4}))\right|\\
    &\leq \delta+\frac{C_1}{C_n^4}(\log\zeta(\mathbf{R},e_{n,1},\delta))^{1/2}\max_{R\in R_\delta} \left\{E_{\epsilon_{i_1}}\left[\sum\limits_{1\leq i_1 < i_2 < i_3 < i_4\leq n}\epsilon_{i_1}\tilde{h}_R((x_{i_1},y_{i_1}),\cdots,(x_{i_4},y_{i_4}))\right]^2\right\}^{1/2}\\
    &\leq \delta +\frac{C_1}{C_n^4}(\log\zeta(\mathbf{R},e_{n,1},\delta))^{1/2}\max_{R\in R_\delta} \left[\sum_{i_1=1}^n\sum_{i_2<i_3<i_4}\tilde{h}_R((x_{i_1},y_{i_1}),\cdots,(x_{i_4},y_{i_4}))^2\right]^{1/2}\\
    &\leq \delta+\frac{C_1}{C_n^4}(\log\zeta(\mathbf{R},e_{n,1},\delta))^{1/2}\left[\frac{n(n!)^2}{\{(n-3)!\}^2}\right]^{1/2}||\tilde{h}_R||_\infty\\
    &\leq \delta + C_1C_3\mathcal{B}(\log\zeta(R,e_{n,1},\delta))^{1/2}/\sqrt{n}
\end{align*}
The first inequality follows from the triangle inequality and the second inequality uses Lemma \ref{lemma 6}. The third and the fourth inequalities follow after some algebra. The fifth inequality follows from the boundedness of $\tilde{h}_R$ and $\mathcal{B}$. Denote $\mathbf{R}_1$ as the ReLU neural network class $R:\mathbb{R}^p\rightarrow \mathbb{R}$ and its depth $\mathcal{D}$, width $\mathcal{W}$, size $\mathcal{S}$ and boundary $\mathcal{B}$ are same as $\mathbf{R}_{\mathcal{D,W,S,B}}$. We have $\zeta(\mathbf{R},e_{n,1},\delta)<\zeta(\mathbf{R},e_{n,\infty},\delta)$  and $\zeta(\mathbf{R},e_{n,\infty},\delta) <\zeta^d(\mathbf{R}_1,e_{n,\infty},\delta)$. Due to the relationship between empirical distance and VC-dimension of $\mathbf{R}_1$ \citep{anthony1999neural},
\begin{align}\label{anthony}
    \log \zeta(\mathbf{R}_1,e_{n,\infty},\delta)\leq \text{VC}_{\mathbf{R}_1}\log \frac{2e\mathcal{B}n}{\delta\text{VC}_{\mathbf{R}_1}},
\end{align}
and the bound of VC-dimension for $\mathbf{R}_1$ \citep{bartlett2019nearly},
\begin{align}\label{bartlett}
    C_5\mathcal{DS}\log \mathcal{S}\leq \text{VC}_{\mathbf{R}_1} \leq C_6 \mathcal{DS}\log \mathcal{S},
\end{align}
we have
\begin{align*}
    (\log\zeta(\mathbf{R},e_{n,1},\delta))^{1/2}/\sqrt{n}&<\sqrt{d}(\log\zeta(\mathbf{R}_1,e_{n,\infty},\delta))^{1/2}/\sqrt{n}\\
    &\leq \sqrt{d}\left(\text{VC}_{\mathbf{R}_1}\log\frac{2e\mathcal{B}n}{\delta \text{VC}_{\mathbf{R}_1}}\right)^{1/2}/\sqrt{n}\\
    &\leq \sqrt{d}\left(\mathcal{DS}\log \mathcal{S}\log \frac{\mathcal{B}n}{\delta\mathcal{DS}\log\mathcal{S}}\right)^{1/2}/\sqrt{n}.
\end{align*}
Based on the above discussions and setting $\delta$ as $\frac{1}{n}$, with assumption \ref{Ass1} and \ref{Ass3}, we have
\begin{align*}
    \mathbb{E}\sup_{R\in \mathbf{R}_{\mathcal{D,W,S,B}}} | \mathbb{V}_n[R,Y]-\mathbb{V}[R,Y]| &\leq C_2/n + C_1C_2C_3C_4\mathcal{B}\sqrt{d}n^{\frac{-\beta}{2\beta+p}}\\
    &=\mathcal{O}(n^{\frac{-\beta}{2\beta+p}}).
\end{align*}
\end{proof}
\begin{lemma}\label{lemma 7}
   Suppose assumption \ref{Ass3} holds, given n i.i.d samples $\{x_i\}_{i=1}^n$ and $\{u_i\}_{i=1}^n$ which are sampled from $N(0,I_d)$, we have  $$\mathbb{E}\sup_{R\in\mathbf{R}_{\mathcal{D,W,S,B}}}\big| \mathbb{D}_n[R,\gamma_d]-\mathbb{D}[R,\gamma_d]\big|\leq\mathcal{O}(n^{\frac{-\beta}{2\beta+p}})$$
\end{lemma}
\begin{proof}
    Define the kernel function $\tilde{g}_R((x_i,u_i),(x_j,u_j))$ as
    \begin{align*}
        \tilde{g}_R = ||R(x_i)-u_j|| + ||R(x_j)-u_i|| - ||R(x_i)-R(x_j)||-||u_i-u_j||-\mathbb{D}[R,\gamma_d],
    \end{align*}
\end{proof}
and it is easy to check that
\begin{align*}
    \mathbb{D}_n[R,\gamma_d]-\mathbb{D}[R,\gamma_d] = \frac{1}{C_n^2}\sum_{1\leq i,j\leq n \atop i\neq j}\tilde{g}_R((x_i,u_i),(x_j,u_j)).
\end{align*}
By the symmetrization randomization Theorem 3.5.3 in \cite{de2012decoupling},
\begin{align}\label{gr1}
    &\mathbb{E}\sup_{R\in \mathbb{R}_\mathcal{D,W,S,B}}| \mathbb{D}^2_n[R,\gamma_d]-\mathbb{D}^2[R,\gamma_d]| \nonumber\\
    &\leq C_2\mathbb{E}\sup_{R\in \mathbb{R}_\mathcal{D,W,S,B}}\big |\frac{1}{C_n^2}\sum\limits_{1\leq i,j\leq n \atop i\neq j}\epsilon_i\tilde{g}_R((x_i,u_i),(x_j,u_j))\big|,
\end{align}
where $\epsilon_i,i=1,...,n$ are i.i.d Rademacher variables that are independent with samples. $\forall {R \in \mathbf{R}_\mathcal{D,W,S,B}},$
we define a random empirical measure for $R$ and $\tilde{R}$,
$$
\tilde{e}_{n,1}(R, \tilde{R})= \mathbb{E}_{\epsilon_i, i = 1,...,n}\big |\frac{1}{C_n^2}\sum\limits_{1\leq i,j\leq n \atop i\neq j}\epsilon_i(\tilde{g}_R-\tilde{g}_{\tilde{R}})((x_i,u_i),(x_j,u_j))\big|$$

Condition on $\{x_i\}_{i=1,\ldots, n}$ and $\{u_i\}_{i=1}^n$, let $\zeta(\mathbf{R}, \tilde{e}_{n,1}, \delta)$ be  the covering number of the neural network class $\mathbf{R}$ with respect  to the  empirical  distance  $\tilde{e}_{n,1}$ at scale of $\delta>0$.
With Lemma \ref{lemma 6}, we build the connection between the expectation \eqref{gr1} and the covering set $\mathbf{R}_\delta$. With covering set $\tilde{\mathbf{R}}_\delta$ and covering number $\zeta({\mathbf{R},\tilde{e}_{n,1},\delta})$, then
\begin{align*}
    &E_{\epsilon_i}[\sup_{R\in \mathbf{R}_\mathcal{D,W,S,B}}|\frac{1}{C_n^2}\sum_{1\leq i<j\leq n }\epsilon_i\tilde{g}_R((x_i,y_i),(x_j,y_j))|\\
    &\leq \delta +E_{\epsilon_i}[\sup_{R\in \tilde{\mathbf{R}}_\delta}\vert\frac{1}{C_n^2}\sum_{1\leq i<j\leq n }\epsilon_i\tilde{g}_R((x_i,y_i),(x_j,y_j))\\
    &\leq \delta+\frac{C_7}{C_n^2}(\log \zeta(\mathbf{R},e_{n,\infty})^{1/2},\delta)\sqrt{n(n-1)/2}||\tilde{g}_R||_\infty\\
    &\leq \delta +C_7C_8\mathcal{B}\sqrt{d}(\mathcal{DS}\log \mathcal{S}\log \frac{\mathcal{B}n}{\delta\mathcal{DS}\log\mathcal{S}})^{1/2}/\sqrt{n}.
\end{align*}
The first inequality follows from the triangle inequality and the second inequality uses Lemma \ref{lemma 6} and some algebra. The last inequality holds due to inequality \eqref{anthony} and \eqref{bartlett}.
Based on the discussions and Assumption \ref{Ass3}, we set $\delta$ to $1/n$ and we have the final result
\begin{align*}
    \mathbb{E}\sup_{R\in\mathbf{R}_{\mathcal{D,W,S,B}}}\big| \mathbb{D}_n[R,\gamma_d]-\mathbb{D}[R,\gamma_d]\big|&\leq C_2/n + C_2C_7C_8\mathcal{B}\sqrt{d}n^{\frac{-\beta}{2\beta+p}}\\
    &=\mathcal{O}(n^{\frac{-\beta}{2\beta+p}}).
\end{align*}

\begin{lemma}\label{lemma 9}
    Suppose assumption \ref{Ass1} and \ref{Ass3} hold, then \textbf{the Statistical Error}
\begin{align*}
    &\mathbb{E}\sup_{R\in\mathbf{R}_{\mathcal{D,W,S,B}}}|\mathcal{L}(R)-\mathcal{L}_n(R)|]\\
    &\leq \frac{1}{2}\mathbb{E}\sup_{R\in\mathbf{R}_{\mathcal{D,W,S,B}}}|\mathbb{V}[R,Y]-\mathbb{V}_n[R,Y]|+\frac{1}{2}\mathbb{E}\sup_{R\in\mathbf{R}_{\mathcal{D,W,S,B}}}|\mathbb{V}[R,A]-\mathbb{V}_n[R,A]|\\
    &+\mathbb{E}\sup_{R\in\mathbf{R}_{\mathcal{D,W,S,B}}}|\mathbb{D}_n[R,\gamma_d]-\mathbb{D}[R,\gamma_d]\\
    &\leq \mathcal{O}(n^{\frac{-\beta}{2\beta+p}}).
\end{align*}
\end{lemma}
With the above analysis, the no-asymptotic bound for excess risk can be obtained by applying \textbf{Lemma \ref{lemma 5}} and \textbf{Lemma \ref{lemma 9}}, that is
 \begin{align*}
     \mathbb{E}_{ \{X_i,Y_i,A_i\}_{i=1}^n}\{L(\hat{R})-L(R_0)\} &\leq 2\sup_{R\in\mathbf{R}_{\mathcal{D,W,S,B}}}|\mathcal{L}(R)-\mathcal{L}_n(R)| + \inf_{R\in\mathbf{R}_{\mathcal{D,W,S,B}}}|\mathcal{L}(R)-\mathcal{L}(T_{B_n}R_0)|\\
     &\leq \mathcal{O}(n^{\frac{-\beta}{2\beta+p}}).
 \end{align*}

 \section{Details of Experiments}
In this section, we give the details of the network structure and hyper-parameters in our experiments. Denote the size of training data by $N_{train}$, the size of validation data by $N_{val}$ and the size of test data by $N_{test}$. Denote the dimension of the feature $X$ by $p$, the dimension of representation by $d$, the batch size by $bs$, the learning rate by $lr$, the learning rate decay step by $lr_{step}$ and the decay rate by $dr$. We consider the Adam algorithm \citep{kingma2014adam} for the optimization
for the neural network with the PyTorch package in python.

For the simulated examples mentioned in Section 5 of the main text and the additional examples, we used the following network structures and parameters shown in Table \ref{sptable1} and Table \ref{sptable2}.
\begin{table}[H]
\label{sptable1}
\centering
    \caption{Hyper-parameters for simulated examples.}
    \resizebox{\textwidth}{!}{
\begin{tabular}{lcccccccccc}
    \hline
     & $N_{train}$ & $N_{val}$ &$N_{test}$ & $p$ & $d$ & $bs$ & $lr$ & $lr_{step}$ &$dr$  & Epoch\\
    \hline
   Linear & 10000 & 1000 & 1000 &50 & 8 &  128 & 8$\times$1e-4 & 20 &0.5& 100\\
   Non-linear & 10000 & 1000 & 1000 &50 & 8 &  128 & 8$\times$1e-4 & 20 &0.5& 100\\
    \hline
    \end{tabular}
    }
    \label{Table:parametersSimu}
\end{table}

\begin{table}[H]
\label{sptable2}
\centering
    \caption{MLP architectures for the representation in simulation examples.}
    \resizebox{0.8\textwidth}{!}{
\begin{tabular}{llccc}
    \hline
    & Layers & Details & Input size & Output size \\
    \hline
Linear & Layer 1 & Linear & 50 & 8 \\
\hline
\multirow{3}{*}{Non-linear} & Layer 1 & Linear & 50 & 32 \\
&Activation &  ReLU & 32 & 32  \\
&Layer 2 & Linear & 32 & 8  \\
    \hline
    \end{tabular}
    }
    \label{Table:MLP }
\end{table}
For the real-world datasets mentioned in Section 6 of the main text, we present the network structures and hyper-parameters in Tables \ref{sptable3} and \ref{sptable4}.

\begin{table}[H]\label{sptable3}
\centering
\caption{Hyper-parameters for real-world datasets.}
    \resizebox{\textwidth}{!}{
\begin{tabular}{lcccccccccc}
    \hline
     Dataset & $N_{train}$ & $N_{val}$ &$N_{test}$ & $p$ & $d$ & $bs$ & $lr$ & $lr_{step}$ &$dr$  & Epoch\\
    \hline
   UCI Adult & 38097 & 4762 & 4762 & 106 & 8 &  128 & 8$\times$1e-4 & 20 &0.5& 60\\
   Heritage Health & 44740 & 5592 & 5592 & 71 & 8 &  128 & 8$\times$1e-4 & 20 &0.5& 60\\
   Bias-in-Bios & 255710 & 39369 & 98344 & 768 & 768 & 128 & 4$\times$1e-4 & 20 & 0.5 &60\\
    \hline
    \end{tabular}}
    \label{Table:parametersReal}
\end{table}

\begin{table}[H]\label{sptable4}
\centering
\caption{MLP architectures for the representation in real-world datasets.}
    \resizebox{\textwidth}{!}{
\begin{tabular}{llccc}
    \hline
   Dataset & Layers & Details & Input size & Output size \\
    \hline
UCI Adult & Layer 1 & Linear & 106 & 8 \\
\hline
\multirow{3}{*}{Heritage Health} & Layer 1 & Linear & 71 & 128 \\
&Activation &  ReLU & 128 & 128  \\
&Layer 2 & Linear & 128 & 8  \\
    \hline
\multirow{5}{*}{Bias-in-Bios} & Layer 1 & Linear & 768 & 768 \\
& Activation & LeakyReLU(0.2) & 768 & 768\\
& Layer 2 & Linear & 768 & 768 \\
& Activation & LeakyReLU(0.2) & 768 & 768\\
& Layer 3 &Linear & 768 & 768\\
\hline
    \end{tabular}
    }
    \label{Table:MLP }
\end{table}

\bibliographystyle{chicago}
\begin{singlespace}
\bibliography{SFLarXiv0302.bib}

\begin{thebibliography}{}

\bibitem[\protect\citeauthoryear{Anthony, Bartlett, Bartlett, et~al.}{Anthony
  et~al.}{1999}]{anthony1999neural}
Anthony, M., P.~L. Bartlett, P.~L. Bartlett, et~al. (1999).
\newblock {\em Neural network learning: Theoretical foundations}, Volume~9.
\newblock cambridge university press Cambridge.

\bibitem[\protect\citeauthoryear{Anthony~Goldbloom}{Anthony~Goldbloom}{2011}]{hhp}
Anthony~Goldbloom, B.~H. (2011).
\newblock Heritage health prize.

\bibitem[\protect\citeauthoryear{Arjovsky, Bottou, Gulrajani, and
  Lopez-Paz}{Arjovsky et~al.}{2019}]{arjovsky2019invariant}
Arjovsky, M., L.~Bottou, I.~Gulrajani, and D.~Lopez-Paz (2019).
\newblock Invariant risk minimization.
\newblock {\em arXiv preprint arXiv:1907.02893\/}.

\bibitem[\protect\citeauthoryear{Bartlett, Harvey, Liaw, and
  Mehrabian}{Bartlett et~al.}{2019}]{bartlett2019nearly}
Bartlett, P.~L., N.~Harvey, C.~Liaw, and A.~Mehrabian (2019).
\newblock Nearly-tight vc-dimension and pseudodimension bounds for piecewise
  linear neural networks.
\newblock {\em Journal of Machine Learning Research\/}~{\em 20\/}(63), 1--17.

\bibitem[\protect\citeauthoryear{Becker and Kohavi}{Becker and
  Kohavi}{1996}]{misc_adult_2}
Becker, B. and R.~Kohavi (1996).
\newblock {Adult}.
\newblock UCI Machine Learning Repository.
\newblock {DOI}: https://doi.org/10.24432/C5XW20.

\bibitem[\protect\citeauthoryear{Beutel, Chen, Zhao, and Chi}{Beutel
  et~al.}{2017}]{beutel2017data}
Beutel, A., J.~Chen, Z.~Zhao, and E.~H. Chi (2017).
\newblock Data decisions and theoretical implications when adversarially
  learning fair representations.
\newblock {\em arXiv preprint arXiv:1707.00075\/}.

\bibitem[\protect\citeauthoryear{Chen, Kornblith, Norouzi, and Hinton}{Chen
  et~al.}{2020}]{chen2020simple}
Chen, T., S.~Kornblith, M.~Norouzi, and G.~Hinton (2020).
\newblock A simple framework for contrastive learning of visual
  representations.
\newblock In {\em International conference on machine learning}, pp.\
  1597--1607. PMLR.

\bibitem[\protect\citeauthoryear{Chen, Jiao, Qiu, and Yu}{Chen
  et~al.}{2024}]{chen2024deep}
Chen, Y., Y.~Jiao, R.~Qiu, and Z.~Yu (2024).
\newblock Deep nonlinear sufficient dimension reduction.
\newblock {\em The Annals of Statistics\/}~{\em 52\/}(3), 1201--1226.

\bibitem[\protect\citeauthoryear{Cook and Ni}{Cook and
  Ni}{2005}]{cook2005sufficient}
Cook, R.~D. and L.~Ni (2005).
\newblock Sufficient dimension reduction via inverse regression: A minimum
  discrepancy approach.
\newblock {\em Journal of the American Statistical Association\/}~{\em
  100\/}(470), 410--428.

\bibitem[\protect\citeauthoryear{De-Arteaga, Romanov, Wallach, Chayes, Borgs,
  Chouldechova, Geyik, Kenthapadi, and Kalai}{De-Arteaga
  et~al.}{2019a}]{dearteaga2019bias}
De-Arteaga, M., A.~Romanov, H.~Wallach, J.~Chayes, C.~Borgs, A.~Chouldechova,
  S.~Geyik, K.~Kenthapadi, and A.~Kalai (2019a).
\newblock Bias in bios: A case study of semantic representation bias in a
  high-stakes setting.
\newblock In {\em Proceedings of the Conference on Fairness, Accountability,
  and Transparency}, pp.\  120--128. ACM.

\bibitem[\protect\citeauthoryear{De-Arteaga, Romanov, Wallach, Chayes, Borgs,
  Chouldechova, Geyik, Kenthapadi, and Kalai}{De-Arteaga
  et~al.}{2019b}]{de2019bias}
De-Arteaga, M., A.~Romanov, H.~Wallach, J.~Chayes, C.~Borgs, A.~Chouldechova,
  S.~Geyik, K.~Kenthapadi, and A.~T. Kalai (2019b).
\newblock Bias in bios: A case study of semantic representation bias in a
  high-stakes setting.
\newblock In {\em proceedings of the Conference on Fairness, Accountability,
  and Transparency}, pp.\  120--128.

\bibitem[\protect\citeauthoryear{De~la Pena and Gin{\'e}}{De~la Pena and
  Gin{\'e}}{2012}]{de2012decoupling}
De~la Pena, V. and E.~Gin{\'e} (2012).
\newblock {\em Decoupling: from dependence to independence}.
\newblock Springer Science \& Business Media.

\bibitem[\protect\citeauthoryear{Devlin, Chang, Lee, and Toutanova}{Devlin
  et~al.}{2018}]{devlin2018bert}
Devlin, J., M.-W. Chang, K.~Lee, and K.~Toutanova (2018).
\newblock Bert: Pre-training of deep bidirectional transformers for language
  understanding.
\newblock {\em arXiv preprint arXiv:1810.04805\/}.

\bibitem[\protect\citeauthoryear{Edwards and Storkey}{Edwards and
  Storkey}{2015}]{edwards2015censoring}
Edwards, H. and A.~Storkey (2015).
\newblock Censoring representations with an adversary.
\newblock {\em arXiv preprint arXiv:1511.05897\/}.

\bibitem[\protect\citeauthoryear{Fukumizu, Bach, and Jordan}{Fukumizu
  et~al.}{2009}]{fukumizu2009kernel}
Fukumizu, K., F.~R. Bach, and M.~I. Jordan (2009).
\newblock Kernel dimension reduction in regression.
\newblock {\em The Annals of Statistics\/}~{\em 37\/}(4), 1871--1905.

\bibitem[\protect\citeauthoryear{Glorot and Bengio}{Glorot and
  Bengio}{2010}]{glorot2010understanding}
Glorot, X. and Y.~Bengio (2010).
\newblock Understanding the difficulty of training deep feedforward neural
  networks.
\newblock In {\em Proceedings of the thirteenth international conference on
  artificial intelligence and statistics}, pp.\  249--256. JMLR Workshop and
  Conference Proceedings.

\bibitem[\protect\citeauthoryear{Gretton, Borgwardt, Rasch, Sch{\"o}lkopf, and
  Smola}{Gretton et~al.}{2012}]{gretton2012kernel}
Gretton, A., K.~M. Borgwardt, M.~J. Rasch, B.~Sch{\"o}lkopf, and A.~Smola
  (2012).
\newblock A kernel two-sample test.
\newblock {\em The Journal of Machine Learning Research\/}~{\em 13\/}(1),
  723--773.

\bibitem[\protect\citeauthoryear{Guo, Wang, Wang, and Zha}{Guo
  et~al.}{2022}]{guo2022learning}
Guo, D., C.~Wang, B.~Wang, and H.~Zha (2022).
\newblock Learning fair representations via distance correlation minimization.
\newblock {\em IEEE Transactions on Neural Networks and Learning Systems\/}.

\bibitem[\protect\citeauthoryear{Gupta, Ferber, Dilkina, and Ver~Steeg}{Gupta
  et~al.}{2021}]{gupta2021controllable}
Gupta, U., A.~M. Ferber, B.~Dilkina, and G.~Ver~Steeg (2021).
\newblock Controllable guarantees for fair outcomes via contrastive information
  estimation.
\newblock In {\em Proceedings of the AAAI Conference on Artificial
  Intelligence}, Volume~35, pp.\  7610--7619.

\bibitem[\protect\citeauthoryear{Han, Baldwin, and Cohn}{Han
  et~al.}{2021}]{han2021diverse}
Han, X., T.~Baldwin, and T.~Cohn (2021).
\newblock Diverse adversaries for mitigating bias in training.
\newblock In {\em Proceedings of the 16th Conference of the European Chapter of
  the Association for Computational Linguistics: Main Volume}, pp.\
  2760--2765.

\bibitem[\protect\citeauthoryear{Hardt, Price, and Srebro}{Hardt
  et~al.}{2016}]{hardt2016equality}
Hardt, M., E.~Price, and N.~Srebro (2016).
\newblock Equality of opportunity in supervised learning.
\newblock {\em Advances in neural information processing systems\/}~{\em 29}.

\bibitem[\protect\citeauthoryear{He, Fan, Wu, Xie, and Girshick}{He
  et~al.}{2020}]{he2020momentum}
He, K., H.~Fan, Y.~Wu, S.~Xie, and R.~Girshick (2020).
\newblock Momentum contrast for unsupervised visual representation learning.
\newblock In {\em Proceedings of the IEEE/CVF conference on computer vision and
  pattern recognition}, pp.\  9729--9738.

\bibitem[\protect\citeauthoryear{Huang, Jiao, Liao, Liu, and Yu}{Huang
  et~al.}{2024}]{huang2024deep}
Huang, J., Y.~Jiao, X.~Liao, J.~Liu, and Z.~Yu (2024).
\newblock Deep dimension reduction for supervised representation learning.
\newblock {\em IEEE Transactions on Information Theory\/}.

\bibitem[\protect\citeauthoryear{Huo and Sz{\'e}kely}{Huo and
  Sz{\'e}kely}{2016}]{huo2016fast}
Huo, X. and G.~J. Sz{\'e}kely (2016).
\newblock Fast computing for distance covariance.
\newblock {\em Technometrics\/}~{\em 58\/}(4), 435--447.

\bibitem[\protect\citeauthoryear{Jiao, Shen, Lin, and Huang}{Jiao
  et~al.}{2023}]{jiao2023deep}
Jiao, Y., G.~Shen, Y.~Lin, and J.~Huang (2023).
\newblock Deep nonparametric regression on approximate manifolds: Nonasymptotic
  error bounds with polynomial prefactors.
\newblock {\em The Annals of Statistics\/}~{\em 51\/}(2), 691--716.

\bibitem[\protect\citeauthoryear{Kingma and Ba}{Kingma and
  Ba}{2014}]{kingma2014adam}
Kingma, D.~P. and J.~Ba (2014).
\newblock Adam: A method for stochastic optimization.
\newblock {\em arXiv preprint arXiv:1412.6980\/}.

\bibitem[\protect\citeauthoryear{Kingma and Welling}{Kingma and
  Welling}{2013}]{kingma2013auto}
Kingma, D.~P. and M.~Welling (2013).
\newblock Auto-encoding variational bayes.
\newblock {\em arXiv preprint arXiv:1312.6114\/}.

\bibitem[\protect\citeauthoryear{Kumar, Raghunathan, Jones, Ma, and
  Liang}{Kumar et~al.}{2022}]{kumar2022fine}
Kumar, A., A.~Raghunathan, R.~Jones, T.~Ma, and P.~Liang (2022).
\newblock Fine-tuning can distort pretrained features and underperform
  out-of-distribution.
\newblock {\em arXiv preprint arXiv:2202.10054\/}.

\bibitem[\protect\citeauthoryear{Li and Wang}{Li and
  Wang}{2007}]{li2007directional}
Li, B. and S.~Wang (2007).
\newblock On directional regression for dimension reduction.
\newblock {\em Journal of the American Statistical Association\/}~{\em
  102\/}(479), 997--1008.

\bibitem[\protect\citeauthoryear{Li}{Li}{1991}]{li1991sliced}
Li, K.-C. (1991).
\newblock Sliced inverse regression for dimension reduction.
\newblock {\em Journal of the American Statistical Association\/}~{\em
  86\/}(414), 316--327.

\bibitem[\protect\citeauthoryear{Liu, Li, Yao, Xu, Ma, Xu, and Tong}{Liu
  et~al.}{2022}]{liu2022fair}
Liu, J., Z.~Li, Y.~Yao, F.~Xu, X.~Ma, M.~Xu, and H.~Tong (2022).
\newblock Fair representation learning: An alternative to mutual information.
\newblock In {\em Proceedings of the 28th ACM SIGKDD Conference on Knowledge
  Discovery and Data Mining}, pp.\  1088--1097.

\bibitem[\protect\citeauthoryear{Louizos, Swersky, Li, Welling, and
  Zemel}{Louizos et~al.}{2015}]{louizos2015variational}
Louizos, C., K.~Swersky, Y.~Li, M.~Welling, and R.~Zemel (2015).
\newblock The variational fair autoencoder.
\newblock {\em arXiv preprint arXiv:1511.00830\/}.

\bibitem[\protect\citeauthoryear{Ma and Zhu}{Ma and Zhu}{2012}]{ma2012semi}
Ma, Y. and L.~Zhu (2012).
\newblock Semiparametric approach to dimension reduction.
\newblock {\em Journal of the American Statistical Association\/}~{\em
  107\/}(497), 168--179.

\bibitem[\protect\citeauthoryear{Madras, Creager, Pitassi, and Zemel}{Madras
  et~al.}{2018}]{madras2018learning}
Madras, D., E.~Creager, T.~Pitassi, and R.~Zemel (2018).
\newblock Learning adversarially fair and transferable representations.
\newblock In {\em International Conference on Machine Learning}, pp.\
  3384--3393. PMLR.

\bibitem[\protect\citeauthoryear{Neyshabur, Sedghi, and Zhang}{Neyshabur
  et~al.}{2020}]{neyshabur2020being}
Neyshabur, B., H.~Sedghi, and C.~Zhang (2020).
\newblock What is being transferred in transfer learning?
\newblock {\em Advances in Neural Information Processing Systems\/}~{\em 33},
  512--523.

\bibitem[\protect\citeauthoryear{Park, Lee, Lee, Hwang, Kim, and Byun}{Park
  et~al.}{2022}]{park2022fair}
Park, S., J.~Lee, P.~Lee, S.~Hwang, D.~Kim, and H.~Byun (2022).
\newblock Fair contrastive learning for facial attribute classification.
\newblock In {\em Proceedings of the IEEE/CVF Conference on Computer Vision and
  Pattern Recognition}, pp.\  10389--10398.

\bibitem[\protect\citeauthoryear{Ravfogel, Elazar, Gonen, Twiton, and
  Goldberg}{Ravfogel et~al.}{2020}]{ravfogel2020null}
Ravfogel, S., Y.~Elazar, H.~Gonen, M.~Twiton, and Y.~Goldberg (2020).
\newblock Null it out: Guarding protected attributes by iterative nullspace
  projection.
\newblock In {\em Proceedings of the 58th Annual Meeting of the Association for
  Computational Linguistics}, pp.\  7237--7256.

\bibitem[\protect\citeauthoryear{Shen}{Shen}{2020}]{shen2020deep}
Shen, Z. (2020).
\newblock Deep network approximation characterized by number of neurons.
\newblock {\em Communications in Computational Physics\/}~{\em 28\/}(5),
  1768--1811.

\bibitem[\protect\citeauthoryear{Sheng and Yin}{Sheng and
  Yin}{2016}]{sheng2016sufficient}
Sheng, W. and X.~Yin (2016).
\newblock Sufficient dimension reduction via distance covariance.
\newblock {\em Journal of Computational and Graphical Statistics\/}~{\em
  25\/}(1), 91--104.

\bibitem[\protect\citeauthoryear{Shi, Ding, Kong, and Jiang}{Shi
  et~al.}{2024}]{shi2024debiasing}
Shi, E., L.~Ding, L.~Kong, and B.~Jiang (2024).
\newblock Debiasing with sufficient projection: A general theoretical framework
  for vector representations.
\newblock In {\em Proceedings of the 2024 Conference of the North American
  Chapter of the Association for Computational Linguistics: Human Language
  Technologies (Volume 1: Long Papers)}, pp.\  5960--5975.

\bibitem[\protect\citeauthoryear{Suzuki and Sugiyama}{Suzuki and
  Sugiyama}{2013}]{suzuki2013sufficient}
Suzuki, T. and M.~Sugiyama (2013).
\newblock Sufficient dimension reduction via squared-loss mutual information
  estimation.
\newblock {\em Neural computation\/}~{\em 25\/}(3), 725--758.

\bibitem[\protect\citeauthoryear{Sz{\'e}kely and Rizzo}{Sz{\'e}kely and
  Rizzo}{2013}]{szekely2013energy}
Sz{\'e}kely, G.~J. and M.~L. Rizzo (2013).
\newblock Energy statistics: A class of statistics based on distances.
\newblock {\em Journal of Statistical Planning and Inference\/}~{\em 143\/}(8),
  1249--1272.

\bibitem[\protect\citeauthoryear{Sz{\'e}kely, Rizzo, and Bakirov}{Sz{\'e}kely
  et~al.}{2007}]{szekely2007measuring}
Sz{\'e}kely, G.~J., M.~L. Rizzo, and N.~K. Bakirov (2007).
\newblock Measuring and testing dependence by correlation of distances.

\bibitem[\protect\citeauthoryear{Vepakomma, Tonde, and Elgammal}{Vepakomma
  et~al.}{2018}]{vepakomma2018supervised}
Vepakomma, P., C.~Tonde, and A.~Elgammal (2018).
\newblock Supervised dimensionality reduction via distance correlation
  maximization.
\newblock {\em Electronic Journal of Statistics\/}~{\em 12\/}(1), 960--984.

\bibitem[\protect\citeauthoryear{Zemel, Wu, Swersky, Pitassi, and Dwork}{Zemel
  et~al.}{2013}]{zemel2013learning}
Zemel, R., Y.~Wu, K.~Swersky, T.~Pitassi, and C.~Dwork (2013).
\newblock Learning fair representations.
\newblock In {\em International Conference on Machine Learning}, pp.\
  325--333. PMLR.

\bibitem[\protect\citeauthoryear{Zhao and Gordon}{Zhao and
  Gordon}{2022}]{zhao2022inherent}
Zhao, H. and G.~J. Gordon (2022).
\newblock Inherent tradeoffs in learning fair representations.
\newblock {\em Journal of Machine Learning Research\/}~{\em 23\/}(57), 1--26.

\bibitem[\protect\citeauthoryear{Zhen, Meng, Chakraborty, and Singh}{Zhen
  et~al.}{2022}]{zhen2022versatile}
Zhen, X., Z.~Meng, R.~Chakraborty, and V.~Singh (2022).
\newblock On the versatile uses of partial distance correlation in deep
  learning.
\newblock In {\em European Conference on Computer Vision}, pp.\  327--346.
  Springer.

\bibitem[\protect\citeauthoryear{Zhu, Liao, Li, Jiao, Liu, and Lu}{Zhu
  et~al.}{2023}]{zhu2023invariant}
Zhu, J., X.~Liao, C.~Li, Y.~Jiao, J.~Liu, and X.~Lu (2023).
\newblock Invariant and sufficient supervised representation learning.
\newblock In {\em 2023 International Joint Conference on Neural Networks
  (IJCNN)}, pp.\  1--8. IEEE.

\bibitem[\protect\citeauthoryear{Zhu, Zhu, and Feng}{Zhu
  et~al.}{2010}]{zhu2010dimension}
Zhu, L.-P., L.-X. Zhu, and Z.-H. Feng (2010).
\newblock Dimension reduction in regressions through cumulative slicing
  estimation.
\newblock {\em Journal of the American Statistical Association\/}~{\em
  105\/}(492), 1455--1466.

\end{thebibliography}
\end{singlespace}
\end{document}